\documentclass[journal]{IEEEtran}

\usepackage[T1]{fontenc}
\usepackage[utf8]{inputenc}
\usepackage{amsmath}
\usepackage{graphicx}
\usepackage{stfloats}
\usepackage{flafter}
\usepackage{booktabs}
\usepackage{cite}
\usepackage[hidelinks]{hyperref}

\fnbelowfloat
\newcommand{\tablebodyformat}{\normalfont\scriptsize\setlength{\tabcolsep}{0pt}\renewcommand{\arraystretch}{1.10}}
\newcommand{\exhibitnotes}[2]{\par\smallskip\noindent\begin{minipage}{\linewidth} \normalfont#1\setlength{\parindent}{0pt}\textit{Notes:} #2 \end{minipage}\par}
\newcommand{\figurenotes}[1]{\exhibitnotes{\footnotesize}{#1}}
\newcommand{\tablenotes}[1]{\exhibitnotes{\scriptsize}{#1}}

\newtheorem{proposition}{Proposition}

\title{Coincidence-Maximizing Customer Clustering Using Smart Meter Data}

\author{Ramiro H. G\'alvez, Gabriel Martos Venturini, and Iv\'an Mondrzak\thanks{The authors are with Universidad Torcuato Di Tella, Ciudad de Buenos Aires, Argentina. Corresponding authors: Ramiro H. G\'alvez (e-mail: rgalvez@utdt.edu) and Gabriel Martos Venturini (e-mail: gmartos@utdt.edu).}}

\begin{document}
\maketitle
\raggedbottom
\clubpenalty=10000
\widowpenalty=10000
\displaywidowpenalty=10000

\begin{abstract}
Electricity distribution planning depends on peak-demand coincidence within customer groups, yet similarity-based clustering of smart-meter data does not optimize it directly. We formulate customer grouping to maximize the average of group coincidence factors, weighted by each group's daily share of summed individual peak demand. We prove that jointly optimizing memberships and group peak times is strongly NP-hard even for two groups. We therefore develop a fixed-peak refinement procedure applicable to any feasible partition. It alternates between identifying group peak times and exactly reassigning customers under minimum group-size constraints. Each polynomial-time reassignment maximizes a lower bound tight at the current partition, so accepted updates strictly improve weighted coincidence. We initialize refinement by clustering individual peak-time distributions and evaluate the method on four public smart-meter panels (27,654 customers) using chronological holdouts and 200 paired replicates with complete refitting. Refinement consistently improves weighted coincidence during fitting and reduces sensitivity to initialization. On held-out days, the method outperforms random grouping and five standalone benchmarks, with mean paired gains of 1.08--3.48 percentage points over the strongest benchmark. Requiring nearly equal group sizes lowers test weighted coincidence only modestly. The groups retain recognizable peak-time patterns over time, suggesting applications in peak-aware customer aggregation.
\end{abstract}


\section{Introduction}
\label{sec:intro}

\IEEEPARstart{D}{istribution} feeders and transformers must accommodate their customers' coincident peak demand, which depends on both the magnitudes and timing of individual loads \cite{dickert2010residential,sun2019peak}. This dependence is captured by the coincidence factor (CF), the ratio of a group's aggregate peak to the sum of its members' individual peaks \cite{bary1945coincidence,dickert2010residential,motlagh2019clustering}. The CF matters because two groups with similar sums of individual peak demand can reach very different aggregate peaks: the group whose customers tend to peak at the same time reaches the higher one \cite{sun2016diversified,roberts2019apartments}. This distinction is becoming increasingly relevant as heat pumps and electric vehicles introduce high-power loads that may concentrate demand at particular times \cite{love2017heatpump,bollerslev2022coincidence,poore2024newloads}. Interval-level observations from advanced metering infrastructure (AMI) now make it possible to identify recurring overlap in customers' high-demand periods \cite{wang2019smartmeterreview}. Combined with network connectivity and asset information, groups formed around these patterns could help characterize the mix of peak-related demand behaviors among customers served by individual distribution assets.

Most existing smart-meter clustering methods rely on similarity in average daily profiles \cite{mcloughlin2015clustering,yilmaz2019comparison}, behavioral features \cite{haben2016clustering,alotaibi2016features}, or temporal-dependence summaries \cite{alonso2020quantile}. Other approaches accommodate shifts in load timing through elastic dissimilarities such as dynamic time warping (DTW) \cite{yerbury2025comparing}, while more recent studies address scalability through distributed auto-clustering \cite{Vahedi2023} and anchor-graph-based subspace clustering \cite{Lin2026}. Although these methods capture useful aspects of consumption and offer different computational strategies, they do not directly favor high coincidence in synchronized load observations. Average profiles, for instance, can obscure day-to-day variation in peak timing \cite{mcloughlin2015clustering,pullinger2024variation}. A separate strand of work focuses on peaks directly, examining evening demand and peak-time variability \cite{dent2012flexibility}, contributions to system peaks \cite{azaza2017responsibility}, and behavior within predefined peak periods \cite{satremeloy2020peak}. Yuan et al.\ \cite{yuan2020systempeak}, for example, use individual peak-time distributions to classify contributions to system peak demand, while constructing load-pattern clusters from seasonal average profiles. Although this literature establishes the relevance of peak-related information, CF has mostly been used to characterize demand diversity or to evaluate and select among partitions generated by other clustering criteria \cite{roberts2019apartments,motlagh2019clustering}, rather than being optimized directly.

To address this gap, we formulate customer clustering as the maximization of a peak-demand-weighted CF. This criterion weights each group's daily CF by its share of the population's summed individual peak demand. Optimizing this objective is challenging, however, because customer memberships determine when aggregate group loads peak, while these peak times affect which assignments yield higher coincidence. We show that this joint optimization is strongly NP-hard even with two groups. Motivated by this difficulty, we develop a fixed-peak refinement procedure (R) that starts from any feasible partition and enforces prescribed minimum group sizes. The procedure alternates between identifying daily aggregate group peak times and optimizing customer assignments while those times remain fixed. We show that fixing the peak times yields a lower bound on weighted CF that is tight at the current partition. Maximizing this bound over feasible assignments can be solved exactly in polynomial time, and every accepted improvement therefore strictly increases the weighted CF objective. We initialize R with an interpretable partition obtained by clustering customers' empirical distributions of daily peak times using Euclidean $K$-means (P). Combining this initialization with refinement yields the proposed P+R method.

We evaluate P+R on four public smart-meter panels comprising 27,654 customers. To distinguish improvements achieved during fitting from coincidence maintained on later observations, we construct groups using an earlier period and evaluate them on subsequent days without changing customer memberships \cite{ullmann2022validation}. We further assess robustness to changes in customer composition and sampled days through 200 paired resampling replicates with complete refitting. The results show that refinement consistently improves weighted CF during fitting and reduces sensitivity to initialization. These benefits extend to later observations, where P+R achieves higher coincidence than random grouping and standalone clustering benchmarks, with refinement accounting for much of the advantage. The resulting groups also retain recognizable peak-time patterns, while nearly equal group sizes impose only modest losses in test weighted CF.

The paper makes three main contributions:
\begin{enumerate}
\item We formulate customer grouping as the direct maximization of peak-demand-weighted CF and prove that jointly optimizing memberships and group peak times is strongly NP-hard.
\item We develop a fixed-peak refinement procedure applicable to any feasible partition. Its reassignment problem can be solved exactly in polynomial time under minimum group-size constraints, with every accepted update strictly increasing weighted CF.
\item We evaluate the method on four smart-meter panels using chronological holdouts and paired refitting. The results demonstrate improvements from refinement during fitting, higher out-of-time coincidence than standalone clustering benchmarks, persistent peak-time patterns, and modest losses under nearly equal group-size constraints.
\end{enumerate}

The remainder of the paper is organized as follows. Section~\ref{sec:method} develops the objective and proposed method, Section~\ref{sec:experimental} describes the data and evaluation framework, Section~\ref{sec:results} presents the empirical results, and Section~\ref{sec:conclusions} concludes.

\section{Coincidence-oriented customer clustering}
\label{sec:method}

\subsection{Peak-Demand-Weighted Coincidence Objective}
\label{sec:objective}

Let $\mathcal D$ denote the fitting period, let $\ell_{i,d}(t)\geq0$ denote the load of customer $i$ on day $d$ at intraday interval $t\in\mathcal T=\{1,\ldots,m\}$, and let $\mathcal C=\{C_1,\ldots,C_K\}$ denote a partition of the $N$ customers into $K$ labeled groups. For each customer, define its daily peak as $b_{i,d}=\max_{t\in\mathcal T}\ell_{i,d}(t)$, and let $B_d=\sum_{i=1}^{N}b_{i,d}$ denote the sum of individual peaks on day $d$. For any group with $\sum_{i\in C_k}b_{i,d}>0$, the conventional daily CF compares the group's aggregate peak with the sum of its members' individual peaks \cite{bary1945coincidence,dickert2010residential,motlagh2019clustering}:
\begin{equation}
\mathrm{CF}_{k,d}(\mathcal C)
=
\frac{\max_{t\in\mathcal T}\sum_{i\in C_k}\ell_{i,d}(t)}
     {\sum_{i\in C_k} b_{i,d}}.
\label{eq:group_cf}
\end{equation}

While \eqref{eq:group_cf} measures coincidence within an individual group, evaluating a partition requires combining the CFs of all groups. A simple average would give equal weight to groups with very different peak-demand totals and could favor singleton groups---whose CF equals one whenever their demand is positive. To account for these differences, we weight each group's daily CF by its members' share of $B_d$, so that its contribution reflects their summed individual peak demand. For a day with $B_d>0$, the resulting peak-demand-weighted CF is
\begin{equation}
\mathrm{CF}^{\mathrm w}_d(\mathcal C)
=
\sum_{k=1}^{K}
\frac{\sum_{i\in C_k} b_{i,d}}{B_d}
\,\mathrm{CF}_{k,d}(\mathcal C),
\label{eq:weighted_cf}
\end{equation}

where groups with zero summed individual peaks contribute zero.

Substituting \eqref{eq:group_cf} into \eqref{eq:weighted_cf} gives
\begin{equation}
\mathrm{CF}^{\mathrm w}_d(\mathcal C)
=
\frac{1}{B_d}
\sum_{k=1}^{K}
\max_{t\in\mathcal T}
\sum_{i\in C_k}\ell_{i,d}(t).
\label{eq:daily_score}
\end{equation}
This expression shows that daily weighted CF measures the share of summed individual peak demand realized at the groups' respective aggregate peaks. From an engineering perspective, the weighting is meaningful for customer aggregation because groups representing larger amounts of peak demand should have greater influence on the overall measure than those representing smaller amounts. Note, however, that group peaks may occur at different times, so the objective characterizes coincidence within the groups of a partition rather than across the population as a whole.\footnote{When $K=1$, $\mathrm{CF}^{\mathrm w}_d(\mathcal C)$ reduces to the conventional daily CF of the entire customer population.} From a mathematical perspective, the weighting is convenient because it eliminates the group-specific CF denominators, and since $B_d$ is independent of the partition, maximizing daily weighted CF is equivalent to maximizing the sum of aggregate group peaks.

Consistent with previous multi-day applications of the CF \cite{guan2016energyplanning}, we summarize coincidence across days by averaging the daily scores over the positive-demand days $\mathcal D^{+}=\{d\in\mathcal D:B_d>0\}$:
\begin{equation}
\overline{\mathrm{CF}}^{\mathrm w}(\mathcal C;\mathcal D)
=
\frac{1}{|\mathcal D^{+}|}
\sum_{d\in\mathcal D^{+}}
\frac{1}{B_d}
\sum_{k=1}^{K}
\max_{t\in\mathcal T}
\sum_{i\in C_k}\ell_{i,d}(t).
\label{eq:period_score}
\end{equation}
We adopt $\overline{\mathrm{CF}}^{\mathrm w}(\mathcal C;\mathcal D)$ as the clustering objective because it rewards high within-group coincidence averaged across days, while weighting each group by the peak demand it represents. Since the objective alone does not prevent very small groups, we also let the intended aggregation task determine a minimum size $L_k$ for each group, and require $|C_k|\geq L_k$, with integer $L_k\geq1$ and $\sum_{k=1}^{K}L_k\leq N$ to ensure feasibility. For a given $K$ and given minimum sizes, the resulting optimization problem is:

\begin{equation}
\begin{aligned}
\max_{\mathcal C}\quad
& \overline{\mathrm{CF}}^{\mathrm w}(\mathcal C;\mathcal D) \\
\mathrm{s.t.}\quad
& |C_k|\geq L_k,\qquad k=1,\ldots,K.
\end{aligned}
\label{eq:global_problem}
\end{equation}

The objective in \eqref{eq:global_problem} has several useful properties. On every positive-demand day $d\in\mathcal D^{+}$, $0<\mathrm{CF}^{\mathrm w}_d(\mathcal C)\leq1$, so the period average also satisfies $0<\overline{\mathrm{CF}}^{\mathrm w}(\mathcal C;\mathcal D)\leq1$. Moreover, replacing $\ell_{i,d}(t)$ by $c_d\,\ell_{i,d}(t)$ with $c_d>0$ multiplies both the numerator of \eqref{eq:daily_score} and $B_d$ by $c_d$, so $\mathrm{CF}^{\mathrm w}_d$ and $\overline{\mathrm{CF}}^{\mathrm w}$ are unchanged, whereas customer-specific rescalings $c_i\,\ell_{i,d}(t)$ generally alter them. Thus, the objective captures temporal coincidence while accounting for differences in relative demand magnitudes. An important consideration is its dependence on the number of groups. Splitting a group cannot decrease the sum of aggregate group peaks on any day, so $\overline{\mathrm{CF}}^{\mathrm w}(\mathcal C;\mathcal D)$ is nondecreasing under nested refinements and reaches one when every customer forms a singleton group. Consequently, clustering methods should be compared at the same $K$ and under the same minimum group-size constraints, rather than using the objective alone to select $K$.\footnote{This monotonicity is not guaranteed for partitions fitted independently at different values of $K$, since those partitions need not be nested.}

\begin{figure}[!t]
\centering
\includegraphics[width=\columnwidth]{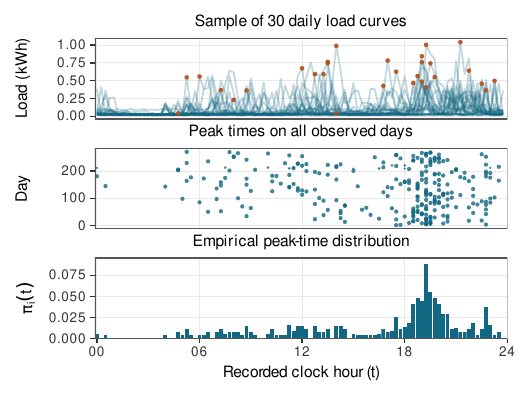}
\caption{Construction of an individual peak-time distribution.}
\label{fig:peak_representation}
\figurenotes{For a CKW customer, the top panel shows 30 daily load curves with peaks marked in orange, while the lower panels summarize peak times over 271 days.}
\end{figure}

\subsection{Global Optimization and Computational Difficulty}
\label{sec:global_problem}

Problem~\eqref{eq:global_problem} is challenging because group memberships and aggregate peak times are jointly determined. Changing a group's membership alters its aggregate load curve and may shift its daily peak, affecting which assignments yield high coincidence. Although the problem admits a mixed-integer linear formulation, finding a global optimum remains computationally difficult, as the following result establishes.

\begin{proposition}
\label{prop:hardness}
The coincidence-oriented clustering problem in \eqref{eq:global_problem} is strongly NP-hard even with $K=2$, two intraday intervals, binary loads, and $L_1=L_2=1$, when the number of observed days is part of the input.
\end{proposition}

\begin{IEEEproof}
We consider the decision version of \eqref{eq:global_problem}, which asks whether some feasible partition achieves weighted CF at or above a prescribed rational threshold. To show that it is strongly NP-hard, we reduce from unweighted Max-Cut \cite{garey1976simplified}. Its decision version asks, for an undirected graph $G=(V,E)$ and an integer $q$, whether $G$ has a cut with at least $q$ crossing edges. We assume $1\leq q\leq |E|$, since the remaining cases are trivial. Given such an instance $(G,q)$, we construct in polynomial time a clustering instance with threshold $1/2+q/(2|E|)$ that admits a feasible partition attaining or exceeding this threshold if and only if $G$ has a cut with at least $q$ crossing edges.

The instance has one customer per vertex and one day per edge, with $K=2$, two intraday intervals, and $L_1=L_2=1$. On the day associated with edge $\{u,v\}$, customer $u$ has profile $(1,0)$, customer $v$ has profile $(0,1)$, and all other customers have profile $(0,0)$. Thus, $B_d=2$ on every day and $\mathcal D^{+}=\mathcal D$. Any partition into two groups corresponds to a cut of $G$---the groups are the two sides, and an edge crosses the cut when its endpoints lie in different groups. If the endpoints of the day's edge share a group, that group's aggregate profile is $(1,1)$ and the other's is $(0,0)$, so the two groups' peaks sum to one. Otherwise, each group has an aggregate peak of one, giving a sum of two. The daily weighted CF is therefore $1/2$ when the endpoints belong to the same group and $1$ otherwise. Averaging across days yields \mbox{$\overline{\mathrm{CF}}^{\mathrm w}(\mathcal C;\mathcal D)=1/2+c_G(\mathcal C)/(2|E|)$}, where $c_G(\mathcal C)$ counts crossing edges.

Since $q\geq1$, any cut with at least $q$ crossing edges has two nonempty sides and hence satisfies the minimum group-size constraints. Conversely, every feasible partition defines a cut. By the formula above, a feasible partition attains weighted CF of at least $1/2+q/(2|E|)$ if and only if $c_G(\mathcal C)\geq q$, which completes the reduction.

The construction uses $2|V||E|$ binary load values and is polynomial in the input size. The load values and minimum-size parameters are bounded by constants, while the threshold has polynomially bounded numerator and denominator. The reduction therefore establishes strong NP-hardness for this restricted case, and hence for the general problem \eqref{eq:global_problem}.
\end{IEEEproof}

\subsection{Two-Stage Coincidence-Oriented Clustering}
\label{sec:two_stage}

To address this difficulty, we propose a two-stage procedure (P+R). In the first stage, clustering of peak-time distributions (P) provides an interpretable initial partition. In the second, fixed-peak refinement (R), our main methodological contribution, improves that partition through exact reassignment.

\subsubsection{Stage 1: Peak-Time Distribution Initialization (P)}

To obtain an initial partition, we represent each customer by the empirical distribution of its daily peak times, capturing when its highest demand typically occurs. For customer $i$, this distribution is defined as
\begin{equation}
\pi_i(t)
=
\frac{1}{|\mathcal D|}
\sum_{d\in\mathcal D}
\frac{\mathbf 1\{\ell_{i,d}(t)=b_{i,d}\}}
     {\sum_{t'\in\mathcal T}\mathbf 1\{\ell_{i,d}(t')=b_{i,d}\}},
\qquad t\in\mathcal T.
\label{eq:peak_distribution}
\end{equation}
Each day contributes one unit of probability mass, divided equally among intervals attaining the daily maximum. Averaging these contributions yields the probability vector $\boldsymbol{\pi}_i=(\pi_i(1),\ldots,\pi_i(m))$, which captures the concentration, dispersion, and possible multimodality of recurring peak times without considering load magnitudes or the remaining shape of daily profiles.\footnote{Yuan et al.\ \cite{yuan2020systempeak} also construct individual peak-time distributions, but use them for classification rather than directly for clustering.} Fig.~\ref{fig:peak_representation} illustrates this construction for one customer.

We then cluster these distributions using Euclidean $K$-means. If a fitted partition violates the minimum group-size constraints, we hold its centroids fixed and recompute assignments by minimizing the same squared-distance criterion subject to those constraints. The selected feasible partition, denoted by $\mathcal C^{(0)}$, serves as the starting point for R.

\subsubsection{Stage 2: Fixed-Peak Refinement (R)}
\label{sec:stage2}

Starting from any feasible partition $\mathcal C^{(0)}$---in our case, the one produced by P, R alternates between identifying aggregate group peak times and reassigning customers while those times remain fixed. The procedure relies on a fixed-peak objective that provides a lower bound on $\overline{\mathrm{CF}}^{\mathrm w}(\mathcal C;\mathcal D)$, is tight at the current partition, and is linear in customer assignments. This structure allows exact reassignment under minimum group-size constraints, with any strict improvement in the bound also increasing the original objective.

\emph{Fixed-peak lower bound.} To make this bound precise, let $h=0,1,\ldots$ denote the iteration index and $\mathcal C^{(h)}$ the current feasible partition. For each group $k$ and day $d\in\mathcal D^{+}$, let $\tau_{k,d}^{(h)}$ denote the selected interval at which the group's aggregate load attains its maximum.\footnote{When several intervals attain the same aggregate peak, our implementation selects one uniformly at random.} These intervals form a $K\times|\mathcal D^{+}|$ array $\boldsymbol{\tau}^{(h)}$ with entries in $\mathcal T$. Given any such array $\boldsymbol{\tau}$, the fixed-peak objective $Q$ evaluates a partition at the prescribed group--day intervals, rather than at its own aggregate peak times:
\begin{equation}
Q(\mathcal C,\boldsymbol{\tau})
=
\frac{1}{|\mathcal D^{+}|}
\sum_{d\in\mathcal D^{+}}
\frac{1}{B_d}
\sum_{k=1}^{K}
\sum_{i\in C_k}
\ell_{i,d}(\tau_{k,d}).
\label{eq:fixed_objective}
\end{equation}

For a fixed partition $\mathcal C$, $Q$ is maximized over $\boldsymbol{\tau}$ by selecting, independently for each group and day, an interval at which the group's aggregate load peaks. Therefore, $\overline{\mathrm{CF}}^{\mathrm w}(\mathcal C;\mathcal D)=\max_{\boldsymbol{\tau}}Q(\mathcal C,\boldsymbol{\tau})$, and problem~\eqref{eq:global_problem} is equivalent to maximizing $Q$ jointly over feasible partitions and peak-time arrays. R exploits this equivalence through block-coordinate ascent, alternately optimizing peak times and customer assignments.

The improvement guarantee follows from evaluating partitions at the current peak times $\boldsymbol{\tau}^{(h)}$. For any feasible partition $\mathcal C$, this gives a lower bound on its weighted CF, i.e., $Q(\mathcal C,\boldsymbol{\tau}^{(h)})\leq \overline{\mathrm{CF}}^{\mathrm w}(\mathcal C;\mathcal D)$, because weighted CF is the maximum of $Q$ over $\boldsymbol{\tau}$. The bound is generally not tight, since $\boldsymbol{\tau}^{(h)}$ need not coincide with the aggregate peak times of $\mathcal C$. For the current partition $\mathcal C^{(h)}$, however, $\boldsymbol{\tau}^{(h)}$ was selected as its own aggregate peak times, so the bound is tight, i.e., $Q(\mathcal C^{(h)},\boldsymbol{\tau}^{(h)})=\overline{\mathrm{CF}}^{\mathrm w}(\mathcal C^{(h)};\mathcal D)$. Consequently, any feasible reassignment $\mathcal C'$ with $Q(\mathcal C',\boldsymbol{\tau}^{(h)})>Q(\mathcal C^{(h)},\boldsymbol{\tau}^{(h)})$ must also strictly increase $\overline{\mathrm{CF}}^{\mathrm w}$:
\begin{equation}
\begin{aligned}
\overline{\mathrm{CF}}^{\mathrm w}(\mathcal C';\mathcal D)
&\geq Q(\mathcal C',\boldsymbol{\tau}^{(h)})\\
&> Q(\mathcal C^{(h)},\boldsymbol{\tau}^{(h)})\\
&= \overline{\mathrm{CF}}^{\mathrm w}(\mathcal C^{(h)};\mathcal D).
\end{aligned}
\label{eq:monotone_guarantee}
\end{equation}

\emph{Exact fixed-peak reassignment.} With peak times fixed, the lower-bound objective can be expressed as a sum of individual assignment scores. For customer $i$ and candidate group $k$, define the score $s_{i,k}^{(h)}=\frac{1}{|\mathcal D^{+}|}\sum_{d\in\mathcal D^{+}}\ell_{i,d}\big(\tau_{k,d}^{(h)}\big)/B_d$, that is, the customer's demand at group $k$'s current peak times, normalized by $B_d$ and averaged over days. Exchanging the order of summation in \eqref{eq:fixed_objective} then gives \mbox{$Q(\mathcal C,\boldsymbol{\tau}^{(h)})=\sum_{k=1}^{K}\sum_{i\in C_k}s_{i,k}^{(h)}$}, so each customer's contribution depends only on the group to which it is assigned. Importantly, the scores reflect realized demand at aggregate group peaks, allowing secondary peaks and other high-demand intervals to contribute even when a customer's daily maximum occurs elsewhere.

This additive structure allows all customer assignments to be optimized simultaneously while enforcing the minimum group sizes. Let $x=(x_{i,k})\in\{0,1\}^{N\times K}$ be the assignment matrix, with $x_{i,k}=1$ if customer $i$ is assigned to group $k$. The reassignment problem is then:
\begin{equation}
\begin{aligned}
\max_{x}\quad
& \sum\nolimits_{i=1}^{N}\sum\nolimits_{k=1}^{K}s_{i,k}^{(h)}x_{i,k} \\
\mathrm{s.t.}\quad
& \textstyle\sum_{k=1}^{K}x_{i,k}=1,
&& i=1,\ldots,N,\\
& \textstyle\sum_{i=1}^{N}x_{i,k}\geq L_k,
&& k=1,\ldots,K,\\
& x_{i,k}\in\{0,1\}.
\end{aligned}
\label{eq:fixed_assignment}
\end{equation}
Problem~\eqref{eq:fixed_assignment} admits a minimum-cost-flow formulation with an integral optimum and can therefore be solved exactly in polynomial time \cite{ahuja1993networkflows}, as in constrained $K$-means \cite{bradley2000constrained}.

Without minimum group-size constraints, each customer would be assigned to its highest-scoring group. Enforcing these constraints, however, may require assigning some customers to lower-scoring groups. We measure the loss from the size constraints by the regret $r_{i,k}^{(h)}=\max_j s_{i,j}^{(h)}-s_{i,k}^{(h)}$ for assigning customer $i$ to group $k$. Since $\sum_{i}\max_j s_{i,j}^{(h)}$ does not depend on the assignment, solving \eqref{eq:fixed_assignment} is equivalent to minimizing the total regret $\sum_{i}\sum_{k}r_{i,k}^{(h)}x_{i,k}$, which amounts to choosing which customers fill each group's mandatory positions. To this end, our implementation creates $L_k$ mandatory positions for each group and matches them to distinct customers through a rectangular assignment problem, solved using a shortest-augmenting-path algorithm \cite{crouse2016rectangular}. The remaining customers are then assigned to their highest-scoring groups, incurring no additional regret. Since every feasible partition must assign at least $L_k$ customers to each group, its total regret cannot be lower than the minimum matching cost. The resulting assignment attains this bound and therefore solves \eqref{eq:fixed_assignment} exactly.

\emph{Monotone refinement.} Let $\widehat{\mathcal C}^{(h)}$ denote the partition induced by an exact solution of \eqref{eq:fixed_assignment}. This candidate is accepted only if its fixed-peak objective exceeds that of the current partition by more than a tolerance $\epsilon\geq0$, set to $10^{-14}$ in the implementation. By \eqref{eq:monotone_guarantee}, every accepted update strictly increases $\overline{\mathrm{CF}}^{\mathrm w}$. Upon acceptance, we set $\mathcal C^{(h+1)}=\widehat{\mathcal C}^{(h)}$ and recompute the group peak times for the next iteration.

If the improvement condition fails, refinement stops at $\mathcal C^{(h)}$. Because reassignment is solved exactly, no feasible partition improves $Q$ at $\boldsymbol\tau^{(h)}$ by more than $\epsilon$, and $\boldsymbol\tau^{(h)}$ already maximizes $Q$ for $\mathcal C^{(h)}$. With $\epsilon=0$, the pair $(\mathcal C^{(h)},\boldsymbol\tau^{(h)})$ is therefore a blockwise optimum: neither block can be improved while the other is held fixed. This is weaker than global optimality, since changing assignments and peak times jointly may still yield a better solution.\footnote{In our implementation, refinement is also limited to 200 iterations. Reaching this limit does not guarantee even blockwise optimality.}

\emph{Computational complexity.} Let $D=|\mathcal D^{+}|$ and $M=\sum_{k=1}^{K}L_k$. For the dense regret-based implementation, one refinement iteration has complexity $O\!\left(ND(m+K)+NM^2\right)$, where the first term accounts for identifying group peak times and computing assignment scores, and the second for solving the rectangular assignment problem. With $D$, $m$, $K$, and $M$ fixed, this bound is linear in $N$. However, when $M$ grows proportionally with $N$, the assignment component yields a cubic bound.

\section{Data and Empirical Evaluation Framework}
\label{sec:experimental}

This section describes the smart-meter panels and the evaluation framework used to assess P+R, including chronological testing, comparison methods, and paired resampling.

\subsection{Smart-Meter Panels and Preprocessing}
\label{sec:data}

We evaluate P+R using four public smart-meter panels comprising 27,654 customers. The panels differ in geographic setting, customer composition, and sampling resolution. To assess performance over time, each panel is divided into an earlier training period $\mathcal D_{\mathrm{tr}}$ of approximately nine months and a subsequent test period $\mathcal D_{\mathrm{te}}$ of approximately three months. Table~\ref{tab:datasets} summarizes the retained customer cohorts and their evaluation windows.

\begin{table}[!t]
\centering
\caption{Smart-meter panels and chronological evaluation windows}
\label{tab:datasets}
\begingroup\tablebodyformat
\begin{tabular*}{\columnwidth}{@{\extracolsep{\fill}}lcrcc@{}}
\toprule
Dataset & Grid & Meters & Training & Test \\
\midrule
GoiEner & 60 & 12,507 & 01/03/20--30/11/20 & 01/12/20--28/02/21 \\
SGSC    & 30 & 6,031  & 01/01/13--30/09/13 & 01/10/13--31/12/13 \\
CKW     & 15 & 4,944  & 02/01/21--30/09/21 & 01/10/21--31/12/21 \\
CAMSL   & 30 & 4,172  & 01/01/18--30/09/18 & 01/10/18--30/12/18 \\
\bottomrule
\end{tabular*}
\endgroup
\tablenotes{Grid denotes the native sampling interval in minutes. CKW excludes March~28 and October~31, 2021, because the source sampling grids on those dates are malformed.}
\end{table}

GoiEner contains Spanish household records \cite{granja2022goiener,quesada2024electricity}, while SGSC comprises Australian residential profiles \cite{roberts2019sgsc}. CKW includes Swiss households and small businesses, both of which are retained in the analysis \cite{sawicki2023ckw,ckwOpenData}. CAMSL contains Japanese residential customers from an experimental tariff study and a nonparticipant sample \cite{camslData,kiguchi2021predicting}. Its retained cohort comprises 578 treatment customers, 348 controls, and 3,246 nonparticipants.

For each panel, we retain only meters with complete, finite, and nonnegative observations on the native sampling grid throughout both predefined periods. All eligible weekdays, weekends, and holidays are included. We preserve source measurements and timestamps, including source-filled observations, without additional imputation or global outlier trimming.

\subsection{Chronological Evaluation Protocol}
\label{sec:validation_protocol}

A partition may achieve high weighted CF during fitting without maintaining the same coincidence on later days. To assess this persistence, we adopt a chronological evaluation protocol \cite{ullmann2022validation}, using the earlier period $\mathcal D_{\mathrm{tr}}$ for fitting and reserving the subsequent period $\mathcal D_{\mathrm{te}}$ for testing, as illustrated in Fig.~\ref{fig:study_design}. For the retained customer cohort, all data-dependent steps, including the construction of customer representations, transformations and dissimilarities, clustering fits, candidate selection, and refinement, use only training observations. Note that, because the periods are consecutive rather than season-matched, this evaluation also incorporates temporal changes in demand patterns.

During testing, customer memberships remain fixed, but each group's aggregate load curve and daily peak are recomputed from the later observations. The resulting test weighted CF measures the coincidence maintained by the fitted groups out of time, whereas training weighted CF is used to assess the coincidence achieved by the different grouping strategies during fitting, particularly the improvements obtained through refinement and the effects of competing initializations.

\begin{figure}[!t]
\centering
\includegraphics[width=\columnwidth]{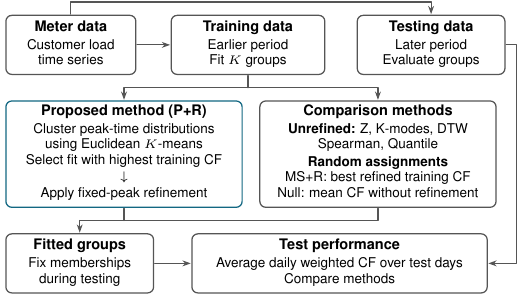}
\caption{Chronological fitting and evaluation for a fixed $K$.}
\label{fig:study_design}
\end{figure}

\subsection{Comparison Methods and Fitting Protocol}
\label{sec:benchmarks}

In this section, we describe the controlled comparisons, literature-based benchmarks, and fitting procedures used to evaluate P+R.

\subsubsection{Controlled Comparisons of Grouping and Refinement}

To distinguish the contributions of initialization and refinement, we compare P with P+R and consider two alternative initializations based on normalized mean profiles and daily peak-time sequences.\footnote{Both alternatives are evaluated before and after applying R, allowing us to examine how the starting representation affects the refined solution.}

\emph{Normalized mean-profile clustering (Z).} Z represents each customer by its typical intraday load shape after removing daily differences in consumption scale. Let $v_{i,d}=\sum_{t\in\mathcal T}\ell_{i,d}(t)$ and define $\widetilde{\ell}_{i,d}(t)=\ell_{i,d}(t)/v_{i,d}$ when $v_{i,d}>0$, or $1/m$ otherwise. Averaging these normalized profiles over the training period gives $z_i(t)=\frac{1}{|\mathcal D_{\mathrm{tr}}|}\sum_{d\in\mathcal D_{\mathrm{tr}}}\widetilde{\ell}_{i,d}(t)$. The vectors $\boldsymbol{z}_i=(z_i(1),\ldots,z_i(m))$ are clustered using Euclidean $K$-means \cite{kwac2014segmentation,yilmaz2019comparison}.

\emph{Daily peak-time sequence clustering (K-modes).} Customers with similar peak-time distributions need not reach their peaks at the same times on the same calendar days. To capture this day-specific alignment, we represent each customer by a categorical sequence of daily peak intervals over the training period.\footnote{When multiple intervals attain the daily maximum, one is selected uniformly at random. An independent sequence realization is generated for each stochastic start and reused across values of $K$.} These sequences are clustered using unweighted $K$-modes \cite{huang1998extensions}, which minimizes mismatches between individual daily peak intervals and the corresponding day-specific modal intervals of their assigned groups. Thus, unlike P---which summarizes peak-time frequencies---K-modes preserves the alignment of individual peaks across calendar days.

\subsubsection{Random-Assignment References}

To assess the contribution of refinement independently of structured clustering, we consider two references based on random feasible partitions. In both cases, customer group labels are sampled independently with equal probability, conditional on satisfying the minimum group-size constraints---yielding a uniform distribution over feasible labeled partitions.

\emph{Random-assignment baseline (Null).} For a given $K$, Null generates 500 independent feasible partitions into $K$ groups, without selection or refinement. Each partition is held fixed across training and testing. For each period, the Null score is the mean weighted CF over the 500 partitions. This provides a reference for assessing whether structured grouping achieves greater coincidence than random assignment at the same $K$.

\emph{Multistart refinement (MS+R).} MS+R generates ten independent random feasible partitions and applies one complete run of R to each using training observations. The refined partition with the highest $\overline{\mathrm{CF}}^{\mathrm w}(\mathcal C;\mathcal D_{\mathrm{tr}})$ is retained for testing. Comparing MS+R with P+R allows us to assess whether peak-time-based initialization provides an additional advantage over random starts after refinement.

\subsubsection{Literature-Based Benchmarks}

We also consider three established clustering approaches that characterize different aspects of load behavior.

\emph{Dynamic time warping clustering (DTW).} Like Z, DTW uses normalized mean profiles $\boldsymbol{z}_i$, but allows similar load shapes to be matched despite temporal shifts. Following Yerbury et al.\ \cite{yerbury2025comparing}, we compute pairwise squared-cost DTW dissimilarities using a symmetric one-hour Sakoe--Chiba window \cite{sakoe1978dtw}. We  cluster these dissimilarities using FasterPAM $k$-medoids \cite{schubert2021fasterpam}, initialized with randomly selected medoids.

\emph{Spearman-correlation clustering (Spearman).} To characterize synchronized dependence across customers, we adapt the approach of Sun et al.\ \cite{sun2019peak}. We define customer dissimilarity as $1-\rho_{ij}$, where $\rho_{ij}$ is the Spearman rank correlation computed from synchronized training observations, using average ranks for ties. We then apply single-linkage hierarchical clustering and cut the resulting hierarchy at $K$. Unlike profile-based methods, this approach groups customers according to dependence in relative demand levels rather than absolute load magnitudes.

\emph{Quantile-autocovariance clustering (Quantile).} The approach of Alonso et al.\ \cite{alonso2020quantile} characterizes temporal dependence within each customer's load variations. Let $u$ index sampling intervals in the original training calendar, and let $\ell_i(u)$ denote customer $i$'s load at interval $u$. The one-day seasonal log differences are then defined as $Y_i(u)=\log(\ell_i(u)+\eta)-\log(\ell_i(u-m)+\eta)$, where $\eta=0$ if all training loads are positive and otherwise equals half the smallest positive training load in the cohort. For each $(\gamma,\gamma')\in\{0.1,0.5,0.9\}^2$, the lag-one quantile autocovariance is the average, over valid consecutive-slot pairs, of the product of two indicators, one for $Y_i(u)$ not exceeding its empirical $\gamma$-quantile and one for $Y_i(u+1)$ not exceeding its empirical $\gamma'$-quantile, less $\gamma\gamma'$. The resulting nine features are clustered using Euclidean complete-linkage hierarchical clustering. For comparability, we cut the hierarchy at the prescribed $K$ and retain all customers---including those in groups smaller than 1\% of the cohort, which Alonso et al.\ exclude from their main-group summaries.

\subsubsection{Fitting and Candidate Selection}

To reduce sensitivity to stochastic initialization, we generate 30 independent fits for P, Z, K-modes, and DTW at each $K>1$. Each method is fitted according to its native clustering objective. Only afterward do we use $\overline{\mathrm{CF}}^{\mathrm w}(\mathcal C;\mathcal D_{\mathrm{tr}})$ to select the feasible partition with the highest unrefined training score. Spearman and Quantile instead produce a single hierarchy and require no multistart search.

The full-cohort evaluation considers $K=1,\ldots,40$ with $L_k=1$. For each $K$, we apply one complete run of R to the selected P partition to obtain P+R and compare all methods at the same number of groups.\footnote{Across eight full-cohort P+R fits (four datasets, $K\in\{20,40\}$, $L_k=1$), initialization with 30 starts and candidate selection took 1.47--3.05~s, while refinement took 1.06--3.10~s over 10--45 accepted updates. Combined fitting times were 3.08--5.47~s, excluding data loading, preprocessing, and feature construction. Runs used an Intel Core i7-10700 under Linux, with four single-threaded initialization workers and six numerical threads for refinement.}

\begin{table*}[t]
\centering
\caption{Training weighted CF before and after refinement}
\label{tab:initialization_refinement}
\begingroup\tablebodyformat
\begin{tabular*}{\textwidth}{@{\extracolsep{\fill}}l@{\hspace{2.8pt}}r@{\hspace{2.8pt}}r@{\hspace{2.8pt}}r@{\hspace{2.8pt}}r@{\hspace{2.8pt}}r@{\hspace{2.8pt}}r@{\hspace{2.8pt}}r@{\hspace{2.8pt}}r@{\hspace{2.8pt}}r@{}}
\toprule
\multicolumn{2}{c}{} & \multicolumn{4}{c}{\textit{A. Before refinement (P)}} & \multicolumn{4}{c}{\textit{B. After refinement (P+R)}} \\
\cmidrule(lr){3-6}\cmidrule(lr){7-10}
Dataset & $K$ & \multicolumn{1}{c}{CF (\%)} & \multicolumn{1}{c}{P $-$ Z} & \multicolumn{1}{c}{P $-$ K-modes} & \multicolumn{1}{c}{P $-$ Null (random)} & \multicolumn{1}{c}{CF (\%)} & \multicolumn{1}{c}{P+R $-$ Z+R} & \multicolumn{1}{c}{P+R $-$ K-modes+R} & \multicolumn{1}{c}{P+R $-$ MS+R} \\
\midrule
GoiEner & 5 & $50.30$ & $\phantom{\mathord{-}}0.93\mkern2.5mu\rlap{$[0.68\mathord{,}1.17]^{*}$}\phantom{[\mathord{-}0.00\mathord{,}0.00]^{*}}$ & $\mathord{-}1.49\mkern2.5mu\rlap{$[\mathord{-}1.77\mathord{,}\mathord{-}1.23]^{*}$}\phantom{[\mathord{-}0.00\mathord{,}\mathord{-}0.00]^{*}}$ & $\phantom{0}5.79\mkern2.5mu\rlap{$[5.51\mathord{,}6.07]^{*}$}\phantom{[00.00\mathord{,}00.00]^{*}}$ & $53.21$ & $\phantom{\mathord{-}}0.19\mkern2.5mu\rlap{$[\mathord{-}0.09\mathord{,}0.48]^{\phantom{*}}$}\phantom{[\mathord{-}0.00\mathord{,}0.00]^{*}}$ & $\mathord{-}0.01\mkern2.5mu\rlap{$[\mathord{-}0.13\mathord{,}0.17]^{\phantom{*}}$}\phantom{[\mathord{-}0.00\mathord{,}0.00]^{*}}$ & $0.08\mkern2.5mu\rlap{$[\mathord{-}0.10\mathord{,}0.40]^{\phantom{*}}$}\phantom{[\mathord{-}0.00\mathord{,}0.00]^{*}}$ \\
 & 10 & $52.78$ & $\phantom{\mathord{-}}1.17\mkern2.5mu\rlap{$[0.96\mathord{,}1.38]^{*}$}\phantom{[\mathord{-}0.00\mathord{,}0.00]^{*}}$ & $\mathord{-}1.45\mkern2.5mu\rlap{$[\mathord{-}1.69\mathord{,}\mathord{-}1.22]^{*}$}\phantom{[\mathord{-}0.00\mathord{,}\mathord{-}0.00]^{*}}$ & $\phantom{0}8.16\mkern2.5mu\rlap{$[7.84\mathord{,}8.44]^{*}$}\phantom{[00.00\mathord{,}00.00]^{*}}$ & $55.84$ & $\mathord{-}0.03\mkern2.5mu\rlap{$[\mathord{-}0.25\mathord{,}0.14]^{\phantom{*}}$}\phantom{[\mathord{-}0.00\mathord{,}0.00]^{*}}$ & $\mathord{-}0.08\mkern2.5mu\rlap{$[\mathord{-}0.32\mathord{,}0.09]^{\phantom{*}}$}\phantom{[\mathord{-}0.00\mathord{,}0.00]^{*}}$ & $0.25\mkern2.5mu\rlap{$[\mathord{-}0.06\mathord{,}0.57]^{\phantom{*}}$}\phantom{[\mathord{-}0.00\mathord{,}0.00]^{*}}$ \\
 & 20 & $54.57$ & $\phantom{\mathord{-}}1.40\mkern2.5mu\rlap{$[1.15\mathord{,}1.62]^{*}$}\phantom{[\mathord{-}0.00\mathord{,}0.00]^{*}}$ & $\mathord{-}1.25\mkern2.5mu\rlap{$[\mathord{-}1.46\mathord{,}\mathord{-}1.03]^{*}$}\phantom{[\mathord{-}0.00\mathord{,}\mathord{-}0.00]^{*}}$ & $\phantom{0}9.73\mkern2.5mu\rlap{$[9.38\mathord{,}10.05]^{*}$}\phantom{[00.00\mathord{,}00.00]^{*}}$ & $57.68$ & $\phantom{\mathord{-}}0.05\mkern2.5mu\rlap{$[\mathord{-}0.17\mathord{,}0.27]^{\phantom{*}}$}\phantom{[\mathord{-}0.00\mathord{,}0.00]^{*}}$ & $\mathord{-}0.14\mkern2.5mu\rlap{$[\mathord{-}0.32\mathord{,}0.06]^{\phantom{*}}$}\phantom{[\mathord{-}0.00\mathord{,}0.00]^{*}}$ & $0.81\mkern2.5mu\rlap{$[0.54\mathord{,}1.04]^{*}$}\phantom{[\mathord{-}0.00\mathord{,}0.00]^{*}}$ \\
 & 40 & $55.87$ & $\phantom{\mathord{-}}1.45\mkern2.5mu\rlap{$[1.30\mathord{,}1.61]^{*}$}\phantom{[\mathord{-}0.00\mathord{,}0.00]^{*}}$ & $\mathord{-}0.87\mkern2.5mu\rlap{$[\mathord{-}1.02\mathord{,}\mathord{-}0.71]^{*}$}\phantom{[\mathord{-}0.00\mathord{,}\mathord{-}0.00]^{*}}$ & $10.60\mkern2.5mu\rlap{$[10.31\mathord{,}10.90]^{*}$}\phantom{[00.00\mathord{,}00.00]^{*}}$ & $59.10$ & $\phantom{\mathord{-}}0.06\mkern2.5mu\rlap{$[\mathord{-}0.05\mathord{,}0.21]^{\phantom{*}}$}\phantom{[\mathord{-}0.00\mathord{,}0.00]^{*}}$ & $\phantom{\mathord{-}}0.06\mkern2.5mu\rlap{$[\mathord{-}0.06\mathord{,}0.19]^{\phantom{*}}$}\phantom{[\mathord{-}0.00\mathord{,}0.00]^{*}}$ & $0.83\mkern2.5mu\rlap{$[0.52\mathord{,}1.09]^{*}$}\phantom{[\mathord{-}0.00\mathord{,}0.00]^{*}}$ \\
\addlinespace[3pt]
SGSC & 5 & $47.16$ & $\phantom{\mathord{-}}1.11\mkern2.5mu\rlap{$[0.96\mathord{,}1.24]^{*}$}\phantom{[\mathord{-}0.00\mathord{,}0.00]^{*}}$ & $\mathord{-}0.69\mkern2.5mu\rlap{$[\mathord{-}0.84\mathord{,}\mathord{-}0.54]^{*}$}\phantom{[\mathord{-}0.00\mathord{,}\mathord{-}0.00]^{*}}$ & $\phantom{0}3.56\mkern2.5mu\rlap{$[3.40\mathord{,}3.70]^{*}$}\phantom{[00.00\mathord{,}00.00]^{*}}$ & $49.47$ & $\mathord{-}0.15\mkern2.5mu\rlap{$[\mathord{-}0.58\mathord{,}0.00]^{\phantom{*}}$}\phantom{[\mathord{-}0.00\mathord{,}0.00]^{*}}$ & $\mathord{-}0.05\mkern2.5mu\rlap{$[\mathord{-}0.47\mathord{,}0.12]^{\phantom{*}}$}\phantom{[\mathord{-}0.00\mathord{,}0.00]^{*}}$ & $0.49\mkern2.5mu\rlap{$[0.00\mathord{,}0.69]^{*}$}\phantom{[\mathord{-}0.00\mathord{,}0.00]^{*}}$ \\
 & 10 & $48.97$ & $\phantom{\mathord{-}}1.32\mkern2.5mu\rlap{$[1.13\mathord{,}1.48]^{*}$}\phantom{[\mathord{-}0.00\mathord{,}0.00]^{*}}$ & $\mathord{-}0.63\mkern2.5mu\rlap{$[\mathord{-}0.81\mathord{,}\mathord{-}0.41]^{*}$}\phantom{[\mathord{-}0.00\mathord{,}\mathord{-}0.00]^{*}}$ & $\phantom{0}5.12\mkern2.5mu\rlap{$[4.91\mathord{,}5.33]^{*}$}\phantom{[00.00\mathord{,}00.00]^{*}}$ & $51.64$ & $\phantom{\mathord{-}}0.09\mkern2.5mu\rlap{$[\mathord{-}0.22\mathord{,}0.37]^{\phantom{*}}$}\phantom{[\mathord{-}0.00\mathord{,}0.00]^{*}}$ & $\phantom{\mathord{-}}0.01\mkern2.5mu\rlap{$[\mathord{-}0.26\mathord{,}0.20]^{\phantom{*}}$}\phantom{[\mathord{-}0.00\mathord{,}0.00]^{*}}$ & $0.67\mkern2.5mu\rlap{$[0.35\mathord{,}1.31]^{*}$}\phantom{[\mathord{-}0.00\mathord{,}0.00]^{*}}$ \\
 & 20 & $50.46$ & $\phantom{\mathord{-}}1.75\mkern2.5mu\rlap{$[1.55\mathord{,}1.95]^{*}$}\phantom{[\mathord{-}0.00\mathord{,}0.00]^{*}}$ & $\mathord{-}0.42\mkern2.5mu\rlap{$[\mathord{-}0.63\mathord{,}\mathord{-}0.23]^{*}$}\phantom{[\mathord{-}0.00\mathord{,}\mathord{-}0.00]^{*}}$ & $\phantom{0}6.18\mkern2.5mu\rlap{$[5.90\mathord{,}6.47]^{*}$}\phantom{[00.00\mathord{,}00.00]^{*}}$ & $53.29$ & $\phantom{\mathord{-}}0.10\mkern2.5mu\rlap{$[\mathord{-}0.05\mathord{,}0.31]^{\phantom{*}}$}\phantom{[\mathord{-}0.00\mathord{,}0.00]^{*}}$ & $\phantom{\mathord{-}}0.02\mkern2.5mu\rlap{$[\mathord{-}0.13\mathord{,}0.21]^{\phantom{*}}$}\phantom{[\mathord{-}0.00\mathord{,}0.00]^{*}}$ & $0.92\mkern2.5mu\rlap{$[0.61\mathord{,}1.18]^{*}$}\phantom{[\mathord{-}0.00\mathord{,}0.00]^{*}}$ \\
 & 40 & $51.58$ & $\phantom{\mathord{-}}1.69\mkern2.5mu\rlap{$[1.50\mathord{,}1.90]^{*}$}\phantom{[\mathord{-}0.00\mathord{,}0.00]^{*}}$ & $\mathord{-}0.22\mkern2.5mu\rlap{$[\mathord{-}0.39\mathord{,}\mathord{-}0.06]^{*}$}\phantom{[\mathord{-}0.00\mathord{,}\mathord{-}0.00]^{*}}$ & $\phantom{0}6.56\mkern2.5mu\rlap{$[6.27\mathord{,}6.86]^{*}$}\phantom{[00.00\mathord{,}00.00]^{*}}$ & $54.70$ & $\phantom{\mathord{-}}0.11\mkern2.5mu\rlap{$[\mathord{-}0.02\mathord{,}0.26]^{\phantom{*}}$}\phantom{[\mathord{-}0.00\mathord{,}0.00]^{*}}$ & $\phantom{\mathord{-}}0.12\mkern2.5mu\rlap{$[\mathord{-}0.03\mathord{,}0.27]^{\phantom{*}}$}\phantom{[\mathord{-}0.00\mathord{,}0.00]^{*}}$ & $0.83\mkern2.5mu\rlap{$[0.59\mathord{,}1.03]^{*}$}\phantom{[\mathord{-}0.00\mathord{,}0.00]^{*}}$ \\
\addlinespace[3pt]
CKW & 5 & $47.56$ & $\mathord{-}0.01\mkern2.5mu\rlap{$[\mathord{-}0.96\mathord{,}0.83]^{\phantom{*}}$}\phantom{[\mathord{-}0.00\mathord{,}0.00]^{*}}$ & $\mathord{-}3.87\mkern2.5mu\rlap{$[\mathord{-}4.74\mathord{,}\mathord{-}3.11]^{*}$}\phantom{[\mathord{-}0.00\mathord{,}\mathord{-}0.00]^{*}}$ & $18.03\mkern2.5mu\rlap{$[17.18\mathord{,}18.74]^{*}$}\phantom{[00.00\mathord{,}00.00]^{*}}$ & $53.62$ & $\phantom{\mathord{-}}0.30\mkern2.5mu\rlap{$[\mathord{-}0.23\mathord{,}0.72]^{\phantom{*}}$}\phantom{[\mathord{-}0.00\mathord{,}0.00]^{*}}$ & $\phantom{\mathord{-}}0.01\mkern2.5mu\rlap{$[\mathord{-}0.35\mathord{,}0.32]^{\phantom{*}}$}\phantom{[\mathord{-}0.00\mathord{,}0.00]^{*}}$ & $0.37\mkern2.5mu\rlap{$[\mathord{-}0.06\mathord{,}0.59]^{\phantom{*}}$}\phantom{[\mathord{-}0.00\mathord{,}0.00]^{*}}$ \\
 & 10 & $52.87$ & $\mathord{-}0.15\mkern2.5mu\rlap{$[\mathord{-}0.83\mathord{,}0.48]^{\phantom{*}}$}\phantom{[\mathord{-}0.00\mathord{,}0.00]^{*}}$ & $\mathord{-}3.32\mkern2.5mu\rlap{$[\mathord{-}4.04\mathord{,}\mathord{-}2.66]^{*}$}\phantom{[\mathord{-}0.00\mathord{,}\mathord{-}0.00]^{*}}$ & $22.97\mkern2.5mu\rlap{$[22.35\mathord{,}23.55]^{*}$}\phantom{[00.00\mathord{,}00.00]^{*}}$ & $59.75$ & $\phantom{\mathord{-}}0.35\mkern2.5mu\rlap{$[\mathord{-}0.49\mathord{,}1.24]^{\phantom{*}}$}\phantom{[\mathord{-}0.00\mathord{,}0.00]^{*}}$ & $\phantom{\mathord{-}}0.40\mkern2.5mu\rlap{$[\mathord{-}0.54\mathord{,}1.50]^{\phantom{*}}$}\phantom{[\mathord{-}0.00\mathord{,}0.00]^{*}}$ & $0.15\mkern2.5mu\rlap{$[\mathord{-}0.71\mathord{,}0.98]^{\phantom{*}}$}\phantom{[\mathord{-}0.00\mathord{,}0.00]^{*}}$ \\
 & 20 & $57.20$ & $\phantom{\mathord{-}}1.21\mkern2.5mu\rlap{$[0.73\mathord{,}1.71]^{*}$}\phantom{[\mathord{-}0.00\mathord{,}0.00]^{*}}$ & $\mathord{-}2.90\mkern2.5mu\rlap{$[\mathord{-}3.61\mathord{,}\mathord{-}2.10]^{*}$}\phantom{[\mathord{-}0.00\mathord{,}\mathord{-}0.00]^{*}}$ & $26.54\mkern2.5mu\rlap{$[25.94\mathord{,}27.11]^{*}$}\phantom{[00.00\mathord{,}00.00]^{*}}$ & $63.85$ & $\phantom{\mathord{-}}0.29\mkern2.5mu\rlap{$[\mathord{-}0.58\mathord{,}1.24]^{\phantom{*}}$}\phantom{[\mathord{-}0.00\mathord{,}0.00]^{*}}$ & $\phantom{\mathord{-}}0.04\mkern2.5mu\rlap{$[\mathord{-}0.81\mathord{,}0.99]^{\phantom{*}}$}\phantom{[\mathord{-}0.00\mathord{,}0.00]^{*}}$ & $0.03\mkern2.5mu\rlap{$[\mathord{-}0.73\mathord{,}0.66]^{\phantom{*}}$}\phantom{[\mathord{-}0.00\mathord{,}0.00]^{*}}$ \\
 & 40 & $60.33$ & $\phantom{\mathord{-}}1.38\mkern2.5mu\rlap{$[0.93\mathord{,}1.74]^{*}$}\phantom{[\mathord{-}0.00\mathord{,}0.00]^{*}}$ & $\mathord{-}2.20\mkern2.5mu\rlap{$[\mathord{-}2.70\mathord{,}\mathord{-}1.74]^{*}$}\phantom{[\mathord{-}0.00\mathord{,}\mathord{-}0.00]^{*}}$ & $28.26\mkern2.5mu\rlap{$[27.73\mathord{,}28.85]^{*}$}\phantom{[00.00\mathord{,}00.00]^{*}}$ & $66.21$ & $\phantom{\mathord{-}}0.29\mkern2.5mu\rlap{$[\mathord{-}0.10\mathord{,}0.68]^{\phantom{*}}$}\phantom{[\mathord{-}0.00\mathord{,}0.00]^{*}}$ & $\mathord{-}0.09\mkern2.5mu\rlap{$[\mathord{-}0.39\mathord{,}0.27]^{\phantom{*}}$}\phantom{[\mathord{-}0.00\mathord{,}0.00]^{*}}$ & $0.30\mkern2.5mu\rlap{$[\mathord{-}0.05\mathord{,}0.64]^{\phantom{*}}$}\phantom{[\mathord{-}0.00\mathord{,}0.00]^{*}}$ \\
\addlinespace[3pt]
CAMSL & 5 & $57.31$ & $\phantom{\mathord{-}}1.73\mkern2.5mu\rlap{$[1.48\mathord{,}2.11]^{*}$}\phantom{[\mathord{-}0.00\mathord{,}0.00]^{*}}$ & $\mathord{-}0.17\mkern2.5mu\rlap{$[\mathord{-}0.36\mathord{,}\mathord{-}0.01]^{*}$}\phantom{[\mathord{-}0.00\mathord{,}\mathord{-}0.00]^{*}}$ & $\phantom{0}5.13\mkern2.5mu\rlap{$[4.85\mathord{,}5.45]^{*}$}\phantom{[00.00\mathord{,}00.00]^{*}}$ & $59.35$ & $\phantom{\mathord{-}}0.25\mkern2.5mu\rlap{$[\mathord{-}0.05\mathord{,}0.67]^{\phantom{*}}$}\phantom{[\mathord{-}0.00\mathord{,}0.00]^{*}}$ & $\phantom{\mathord{-}}0.14\mkern2.5mu\rlap{$[\mathord{-}0.05\mathord{,}0.29]^{\phantom{*}}$}\phantom{[\mathord{-}0.00\mathord{,}0.00]^{*}}$ & $0.07\mkern2.5mu\rlap{$[\mathord{-}0.10\mathord{,}0.18]^{\phantom{*}}$}\phantom{[\mathord{-}0.00\mathord{,}0.00]^{*}}$ \\
 & 10 & $59.45$ & $\phantom{\mathord{-}}2.47\mkern2.5mu\rlap{$[2.23\mathord{,}2.71]^{*}$}\phantom{[\mathord{-}0.00\mathord{,}0.00]^{*}}$ & $\mathord{-}0.17\mkern2.5mu\rlap{$[\mathord{-}0.42\mathord{,}0.08]^{\phantom{*}}$}\phantom{[\mathord{-}0.00\mathord{,}\mathord{-}0.00]^{*}}$ & $\phantom{0}6.86\mkern2.5mu\rlap{$[6.52\mathord{,}7.23]^{*}$}\phantom{[00.00\mathord{,}00.00]^{*}}$ & $61.53$ & $\phantom{\mathord{-}}0.10\mkern2.5mu\rlap{$[\mathord{-}0.15\mathord{,}0.46]^{\phantom{*}}$}\phantom{[\mathord{-}0.00\mathord{,}0.00]^{*}}$ & $\mathord{-}0.03\mkern2.5mu\rlap{$[\mathord{-}0.23\mathord{,}0.19]^{\phantom{*}}$}\phantom{[\mathord{-}0.00\mathord{,}0.00]^{*}}$ & $0.00\mkern2.5mu\rlap{$[\mathord{-}0.18\mathord{,}0.25]^{\phantom{*}}$}\phantom{[\mathord{-}0.00\mathord{,}0.00]^{*}}$ \\
 & 20 & $61.13$ & $\phantom{\mathord{-}}2.68\mkern2.5mu\rlap{$[2.43\mathord{,}2.92]^{*}$}\phantom{[\mathord{-}0.00\mathord{,}0.00]^{*}}$ & $\phantom{\mathord{-}}0.09\mkern2.5mu\rlap{$[\mathord{-}0.11\mathord{,}0.29]^{\phantom{*}}$}\phantom{[\mathord{-}0.00\mathord{,}\mathord{-}0.00]^{*}}$ & $\phantom{0}7.92\mkern2.5mu\rlap{$[7.50\mathord{,}8.34]^{*}$}\phantom{[00.00\mathord{,}00.00]^{*}}$ & $63.32$ & $\phantom{\mathord{-}}0.22\mkern2.5mu\rlap{$[0.04\mathord{,}0.43]^{*}$}\phantom{[\mathord{-}0.00\mathord{,}0.00]^{*}}$ & $\phantom{\mathord{-}}0.03\mkern2.5mu\rlap{$[\mathord{-}0.10\mathord{,}0.18]^{\phantom{*}}$}\phantom{[\mathord{-}0.00\mathord{,}0.00]^{*}}$ & $0.30\mkern2.5mu\rlap{$[0.12\mathord{,}0.49]^{*}$}\phantom{[\mathord{-}0.00\mathord{,}0.00]^{*}}$ \\
 & 40 & $62.45$ & $\phantom{\mathord{-}}2.51\mkern2.5mu\rlap{$[2.33\mathord{,}2.69]^{*}$}\phantom{[\mathord{-}0.00\mathord{,}0.00]^{*}}$ & $\phantom{\mathord{-}}0.55\mkern2.5mu\rlap{$[0.39\mathord{,}0.73]^{*}$}\phantom{[\mathord{-}0.00\mathord{,}\mathord{-}0.00]^{*}}$ & $\phantom{0}8.25\mkern2.5mu\rlap{$[7.83\mathord{,}8.64]^{*}$}\phantom{[00.00\mathord{,}00.00]^{*}}$ & $64.83$ & $\phantom{\mathord{-}}0.23\mkern2.5mu\rlap{$[0.06\mathord{,}0.41]^{*}$}\phantom{[\mathord{-}0.00\mathord{,}0.00]^{*}}$ & $\phantom{\mathord{-}}0.19\mkern2.5mu\rlap{$[\mathord{-}0.03\mathord{,}0.36]^{\phantom{*}}$}\phantom{[\mathord{-}0.00\mathord{,}0.00]^{*}}$ & $0.41\mkern2.5mu\rlap{$[0.25\mathord{,}0.59]^{*}$}\phantom{[\mathord{-}0.00\mathord{,}0.00]^{*}}$ \\
\bottomrule
\end{tabular*}
\endgroup
\tablenotes{Results summarize 200 paired resampling replicates with complete refitting. CF columns report mean training weighted CF (\%) for P in Panel A and P+R in Panel B; contrasts report mean paired differences in percentage points. Brackets show pointwise 95\% robustness intervals, and asterisks mark intervals excluding zero before rounding.}
\end{table*}

\subsection{Paired Resampling and Robustness Assessment}
\label{sec:resampling}

To assess whether method comparisons remain robust to changes in customer composition and sampled days, we perform 200 paired resampling replicates with complete refitting for $K\in\{5,10,20,40\}$. In each replicate, we select 90\% of customers without replacement and resample days with replacement separately within the training and test periods, preserving their original numbers of days. Each sampled day includes the complete intraday profiles of all selected customers, preserving contemporaneous cross-customer dependence.

The same customer subset and sampled days are used for all methods and values of $K$ within each replicate, ensuring that comparisons are paired. All data-dependent fitting steps, including customer representations, clustering fits, candidate selection, and refinement, are repeated using only the sampled training observations.\footnote{For Quantile, sampled training dates serve as anchors, while seasonal differences and lagged pairs are computed using their original calendar positions.} Customer memberships are then held fixed when evaluating the corresponding sampled test days. Null is likewise recomputed from 500 independent feasible random partitions in each replicate.

This common resampling design allows us to evaluate relative performance through paired differences in test weighted CF. For replicate $b$, methods $a$ and $c$, and number of groups $K$, define
\begin{equation*}
\Delta_{a,c,K}^{(b)}
=
\overline{\mathrm{CF}}^{\mathrm w}
\!\left(
\mathcal C_{K}^{a,(b)};
\mathcal D_{\mathrm{te}}^{(b)}
\right)
-
\overline{\mathrm{CF}}^{\mathrm w}
\!\left(
\mathcal C_{K}^{c,(b)};
\mathcal D_{\mathrm{te}}^{(b)}
\right),
\end{equation*}
where $\mathcal C_K^{a,(b)}$ denotes the partition fitted by method $a$ in replicate $b$ with $K$ groups, and $\mathcal D_{\mathrm{te}}^{(b)}$ is the corresponding multiset of sampled test days. Note that positive differences favor method $a$.\footnote{For comparisons involving Null, its weighted CF is replaced by the average over 500 random partitions.} To assess the magnitude and robustness of these differences, we report their mean and empirical 2.5th and 97.5th percentiles for each dataset, method pair, and $K$, with the latter defining a pointwise 95\% robustness interval.

\begin{figure}[!t]
\centering
\includegraphics[width=\columnwidth]{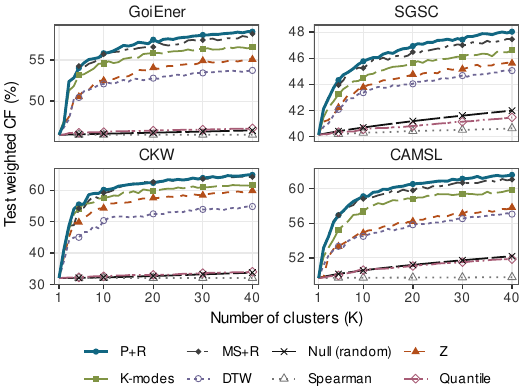}
\caption{Test weighted CF by number of groups.}
\label{fig:cf_performance}
\figurenotes{The curves show test weighted CF for the full cohorts, with customer memberships fitted on training data. Vertical scales differ across datasets.}
\end{figure}

\section{Empirical Results}
\label{sec:results}

This section presents the empirical results, examining the effects of initialization and refinement, out-of-time coincidence and robustness, group composition and peak-time persistence, and sensitivity to minimum group-size requirements.

\subsection{Initialization and Refinement}
\label{sec:results_initialization}

To evaluate P+R purely as an optimization procedure, we start by examining how it improves $\overline{\mathrm{CF}}^{\mathrm w}(\mathcal C;\mathcal D_{\mathrm{tr}})$ and how the refined solutions depend on initialization. Table~\ref{tab:initialization_refinement} compares performance before (Panel A) and after (Panel B) refinement across different starting partitions.
Before refinement, P exceeds the random-assignment baseline Null by 3.56--28.26 percentage points (pp) on average, with all sixteen robustness intervals above zero. However, its standing relative to the structured alternatives is mixed. P achieves a higher mean training weighted CF than Z in fourteen of the sixteen settings, whereas K-modes outperforms P in fourteen---particularly on CKW. These comparisons show that P captures coincidence-relevant structure beyond random grouping, although K-modes more often provides the stronger initial partition.

Refinement substantially improves the objective and reduces these initial differences. R increases P's training weighted CF in every setting, with mean gains of approximately 2.0--6.9 pp. Despite K-modes' initial advantage in most settings, all sixteen P+R $-$ K-modes+R robustness intervals contain zero. Likewise, only two P+R $-$ Z+R intervals exclude zero, both favoring P+R on CAMSL at $K=20$ and $K=40$. These findings suggest that refinement reduces the influence of the starting partition on the attained training objective.

To determine whether structured initialization provides an advantage after refinement, we compare P+R with MS+R. P+R achieves a higher mean training weighted CF than MS+R in fifteen of the sixteen settings, with the remaining difference rounding to zero. Although the gains are generally below one pp, eight robustness intervals exclude zero in favor of P+R.

Overall, the results indicate that R substantially improves training coincidence and reduces sensitivity to initialization, while P provides a modest additional advantage over random starts at relatively low computational cost.

\subsection{Out-of-Time Coincidence and Robustness}
\label{sec:results_performance}

Improvements in $\overline{\mathrm{CF}}^{\mathrm w}(\mathcal C;\mathcal D_{\mathrm{tr}})$ do not necessarily persist on later observations. We therefore examine the coincidence maintained by P+R on the held-out test periods and its robustness to changes in customer composition and sampled days. Fig.~\ref{fig:cf_performance} compares full-cohort test weighted CF across numbers of groups, while Table~\ref{tab:paired_final_methods} reports paired test-period differences under resampling with complete refitting.

\begin{table*}[t]
\centering
\caption{Paired differences in test weighted CF (P+R minus comparator)}
\label{tab:paired_final_methods}
\begingroup\tablebodyformat
\begin{tabular*}{\textwidth}{@{\extracolsep{\fill}}l@{\hspace{4pt}}r@{\hspace{7pt}}r@{\hspace{4pt}}r@{\hspace{4pt}}r@{\hspace{4pt}}r@{\hspace{4pt}}r@{\hspace{4pt}}r@{\hspace{4pt}}r@{}}
\toprule
Dataset & $K$ & \multicolumn{1}{c}{MS+R} & \multicolumn{1}{c}{Null (random)} & \multicolumn{1}{c}{Z} & \multicolumn{1}{c}{K-modes} & \multicolumn{1}{c}{DTW} & \multicolumn{1}{c}{Spearman} & \multicolumn{1}{c}{Quantile} \\
\midrule
GoiEner & 5 & $\phantom{\mathord{-}}0.32\mkern2.5mu\rlap{$[\mathord{-}0.61\mathord{,}1.31]^{\phantom{*}}$}\phantom{[\mathord{-}0.00\mathord{,}0.00]^{*}}$ & $\phantom{0}7.74\mkern2.5mu\rlap{$[7.23\mathord{,}8.34]^{*}$}\phantom{[00.00\mathord{,}00.00]^{*}}$ & $2.95\mkern2.5mu\rlap{$[2.53\mathord{,}3.45]^{*}$}\phantom{[0.00\mathord{,}0.00]^{*}}$ & $1.37\mkern2.5mu\rlap{$[0.83\mathord{,}2.30]^{*}$}\phantom{[0.00\mathord{,}0.00]^{*}}$ & $\phantom{0}3.05\mkern2.5mu\rlap{$[2.55\mathord{,}3.59]^{*}$}\phantom{[0.00\mathord{,}00.00]^{*}}$ & $\phantom{0}7.80\mkern2.5mu\rlap{$[7.29\mathord{,}8.42]^{*}$}\phantom{[00.00\mathord{,}00.00]^{*}}$ & $\phantom{0}7.61\mkern2.5mu\rlap{$[7.01\mathord{,}8.23]^{*}$}\phantom{[00.00\mathord{,}00.00]^{*}}$ \\
 & 10 & $\phantom{\mathord{-}}0.24\mkern2.5mu\rlap{$[\mathord{-}0.24\mathord{,}0.76]^{\phantom{*}}$}\phantom{[\mathord{-}0.00\mathord{,}0.00]^{*}}$ & $\phantom{0}9.36\mkern2.5mu\rlap{$[8.72\mathord{,}9.98]^{*}$}\phantom{[00.00\mathord{,}00.00]^{*}}$ & $2.91\mkern2.5mu\rlap{$[2.52\mathord{,}3.43]^{*}$}\phantom{[0.00\mathord{,}0.00]^{*}}$ & $1.61\mkern2.5mu\rlap{$[1.09\mathord{,}2.03]^{*}$}\phantom{[0.00\mathord{,}0.00]^{*}}$ & $\phantom{0}3.39\mkern2.5mu\rlap{$[2.89\mathord{,}3.84]^{*}$}\phantom{[0.00\mathord{,}00.00]^{*}}$ & $\phantom{0}9.49\mkern2.5mu\rlap{$[8.85\mathord{,}10.10]^{*}$}\phantom{[00.00\mathord{,}00.00]^{*}}$ & $\phantom{0}9.18\mkern2.5mu\rlap{$[8.50\mathord{,}9.81]^{*}$}\phantom{[00.00\mathord{,}00.00]^{*}}$ \\
 & 20 & $\phantom{\mathord{-}}0.72\mkern2.5mu\rlap{$[0.29\mathord{,}1.06]^{*}$}\phantom{[\mathord{-}0.00\mathord{,}0.00]^{*}}$ & $10.74\mkern2.5mu\rlap{$[10.08\mathord{,}11.42]^{*}$}\phantom{[00.00\mathord{,}00.00]^{*}}$ & $3.06\mkern2.5mu\rlap{$[2.65\mathord{,}3.54]^{*}$}\phantom{[0.00\mathord{,}0.00]^{*}}$ & $1.78\mkern2.5mu\rlap{$[1.43\mathord{,}2.15]^{*}$}\phantom{[0.00\mathord{,}0.00]^{*}}$ & $\phantom{0}4.08\mkern2.5mu\rlap{$[3.56\mathord{,}4.61]^{*}$}\phantom{[0.00\mathord{,}00.00]^{*}}$ & $11.02\mkern2.5mu\rlap{$[10.41\mathord{,}11.74]^{*}$}\phantom{[00.00\mathord{,}00.00]^{*}}$ & $10.52\mkern2.5mu\rlap{$[9.86\mathord{,}11.30]^{*}$}\phantom{[00.00\mathord{,}00.00]^{*}}$ \\
 & 40 & $\phantom{\mathord{-}}0.80\mkern2.5mu\rlap{$[0.45\mathord{,}1.09]^{*}$}\phantom{[\mathord{-}0.00\mathord{,}0.00]^{*}}$ & $11.51\mkern2.5mu\rlap{$[10.88\mathord{,}12.15]^{*}$}\phantom{[00.00\mathord{,}00.00]^{*}}$ & $2.94\mkern2.5mu\rlap{$[2.61\mathord{,}3.28]^{*}$}\phantom{[0.00\mathord{,}0.00]^{*}}$ & $2.13\mkern2.5mu\rlap{$[1.85\mathord{,}2.42]^{*}$}\phantom{[0.00\mathord{,}0.00]^{*}}$ & $\phantom{0}4.26\mkern2.5mu\rlap{$[3.90\mathord{,}4.69]^{*}$}\phantom{[0.00\mathord{,}00.00]^{*}}$ & $12.10\mkern2.5mu\rlap{$[11.48\mathord{,}12.71]^{*}$}\phantom{[00.00\mathord{,}00.00]^{*}}$ & $11.21\mkern2.5mu\rlap{$[10.48\mathord{,}11.94]^{*}$}\phantom{[00.00\mathord{,}00.00]^{*}}$ \\
\addlinespace[3pt]
SGSC & 5 & $\phantom{\mathord{-}}0.35\mkern2.5mu\rlap{$[\mathord{-}0.15\mathord{,}0.57]^{\phantom{*}}$}\phantom{[\mathord{-}0.00\mathord{,}0.00]^{*}}$ & $\phantom{0}3.68\mkern2.5mu\rlap{$[3.18\mathord{,}4.02]^{*}$}\phantom{[00.00\mathord{,}00.00]^{*}}$ & $1.87\mkern2.5mu\rlap{$[1.43\mathord{,}2.14]^{*}$}\phantom{[0.00\mathord{,}0.00]^{*}}$ & $1.08\mkern2.5mu\rlap{$[0.63\mathord{,}1.34]^{*}$}\phantom{[0.00\mathord{,}0.00]^{*}}$ & $\phantom{0}2.12\mkern2.5mu\rlap{$[1.67\mathord{,}2.51]^{*}$}\phantom{[0.00\mathord{,}00.00]^{*}}$ & $\phantom{0}3.90\mkern2.5mu\rlap{$[3.41\mathord{,}4.23]^{*}$}\phantom{[00.00\mathord{,}00.00]^{*}}$ & $\phantom{0}3.80\mkern2.5mu\rlap{$[3.31\mathord{,}4.14]^{*}$}\phantom{[00.00\mathord{,}00.00]^{*}}$ \\
 & 10 & $\phantom{\mathord{-}}0.58\mkern2.5mu\rlap{$[0.26\mathord{,}1.10]^{*}$}\phantom{[\mathord{-}0.00\mathord{,}0.00]^{*}}$ & $\phantom{0}4.84\mkern2.5mu\rlap{$[4.39\mathord{,}5.27]^{*}$}\phantom{[00.00\mathord{,}00.00]^{*}}$ & $1.80\mkern2.5mu\rlap{$[1.48\mathord{,}2.13]^{*}$}\phantom{[0.00\mathord{,}0.00]^{*}}$ & $1.42\mkern2.5mu\rlap{$[1.13\mathord{,}1.69]^{*}$}\phantom{[0.00\mathord{,}0.00]^{*}}$ & $\phantom{0}2.21\mkern2.5mu\rlap{$[1.82\mathord{,}2.60]^{*}$}\phantom{[0.00\mathord{,}00.00]^{*}}$ & $\phantom{0}5.28\mkern2.5mu\rlap{$[4.83\mathord{,}5.73]^{*}$}\phantom{[00.00\mathord{,}00.00]^{*}}$ & $\phantom{0}5.08\mkern2.5mu\rlap{$[4.62\mathord{,}5.54]^{*}$}\phantom{[00.00\mathord{,}00.00]^{*}}$ \\
 & 20 & $\phantom{\mathord{-}}0.66\mkern2.5mu\rlap{$[0.42\mathord{,}0.98]^{*}$}\phantom{[\mathord{-}0.00\mathord{,}0.00]^{*}}$ & $\phantom{0}5.45\mkern2.5mu\rlap{$[4.95\mathord{,}5.91]^{*}$}\phantom{[00.00\mathord{,}00.00]^{*}}$ & $1.97\mkern2.5mu\rlap{$[1.64\mathord{,}2.27]^{*}$}\phantom{[0.00\mathord{,}0.00]^{*}}$ & $1.48\mkern2.5mu\rlap{$[1.23\mathord{,}1.72]^{*}$}\phantom{[0.00\mathord{,}0.00]^{*}}$ & $\phantom{0}2.56\mkern2.5mu\rlap{$[2.22\mathord{,}2.88]^{*}$}\phantom{[0.00\mathord{,}00.00]^{*}}$ & $\phantom{0}6.28\mkern2.5mu\rlap{$[5.76\mathord{,}6.70]^{*}$}\phantom{[00.00\mathord{,}00.00]^{*}}$ & $\phantom{0}5.83\mkern2.5mu\rlap{$[5.32\mathord{,}6.26]^{*}$}\phantom{[00.00\mathord{,}00.00]^{*}}$ \\
 & 40 & $\phantom{\mathord{-}}0.63\mkern2.5mu\rlap{$[0.40\mathord{,}0.90]^{*}$}\phantom{[\mathord{-}0.00\mathord{,}0.00]^{*}}$ & $\phantom{0}5.74\mkern2.5mu\rlap{$[5.28\mathord{,}6.18]^{*}$}\phantom{[00.00\mathord{,}00.00]^{*}}$ & $2.01\mkern2.5mu\rlap{$[1.73\mathord{,}2.31]^{*}$}\phantom{[0.00\mathord{,}0.00]^{*}}$ & $1.67\mkern2.5mu\rlap{$[1.44\mathord{,}1.88]^{*}$}\phantom{[0.00\mathord{,}0.00]^{*}}$ & $\phantom{0}2.66\mkern2.5mu\rlap{$[2.33\mathord{,}3.01]^{*}$}\phantom{[0.00\mathord{,}00.00]^{*}}$ & $\phantom{0}7.19\mkern2.5mu\rlap{$[6.73\mathord{,}7.61]^{*}$}\phantom{[00.00\mathord{,}00.00]^{*}}$ & $\phantom{0}6.35\mkern2.5mu\rlap{$[5.89\mathord{,}6.78]^{*}$}\phantom{[00.00\mathord{,}00.00]^{*}}$ \\
\addlinespace[3pt]
CKW & 5 & $\phantom{\mathord{-}}1.20\mkern2.5mu\rlap{$[\mathord{-}0.31\mathord{,}2.08]^{\phantom{*}}$}\phantom{[\mathord{-}0.00\mathord{,}0.00]^{*}}$ & $23.30\mkern2.5mu\rlap{$[22.52\mathord{,}24.12]^{*}$}\phantom{[00.00\mathord{,}00.00]^{*}}$ & $5.58\mkern2.5mu\rlap{$[4.39\mathord{,}6.41]^{*}$}\phantom{[0.00\mathord{,}0.00]^{*}}$ & $1.84\mkern2.5mu\rlap{$[1.17\mathord{,}3.59]^{*}$}\phantom{[0.00\mathord{,}0.00]^{*}}$ & $\phantom{0}9.94\mkern2.5mu\rlap{$[8.63\mathord{,}11.03]^{*}$}\phantom{[0.00\mathord{,}00.00]^{*}}$ & $23.36\mkern2.5mu\rlap{$[22.57\mathord{,}24.19]^{*}$}\phantom{[00.00\mathord{,}00.00]^{*}}$ & $23.06\mkern2.5mu\rlap{$[22.32\mathord{,}23.87]^{*}$}\phantom{[00.00\mathord{,}00.00]^{*}}$ \\
 & 10 & $\phantom{\mathord{-}}1.12\mkern2.5mu\rlap{$[\mathord{-}0.14\mathord{,}2.34]^{\phantom{*}}$}\phantom{[\mathord{-}0.00\mathord{,}0.00]^{*}}$ & $28.01\mkern2.5mu\rlap{$[27.25\mathord{,}28.76]^{*}$}\phantom{[00.00\mathord{,}00.00]^{*}}$ & $5.43\mkern2.5mu\rlap{$[4.62\mathord{,}6.42]^{*}$}\phantom{[0.00\mathord{,}0.00]^{*}}$ & $3.48\mkern2.5mu\rlap{$[2.36\mathord{,}4.82]^{*}$}\phantom{[0.00\mathord{,}0.00]^{*}}$ & $10.28\mkern2.5mu\rlap{$[9.05\mathord{,}12.12]^{*}$}\phantom{[0.00\mathord{,}00.00]^{*}}$ & $28.24\mkern2.5mu\rlap{$[27.47\mathord{,}29.01]^{*}$}\phantom{[00.00\mathord{,}00.00]^{*}}$ & $27.58\mkern2.5mu\rlap{$[26.80\mathord{,}28.36]^{*}$}\phantom{[00.00\mathord{,}00.00]^{*}}$ \\
 & 20 & $\mathord{-}0.14\mkern2.5mu\rlap{$[\mathord{-}0.77\mathord{,}0.45]^{\phantom{*}}$}\phantom{[\mathord{-}0.00\mathord{,}0.00]^{*}}$ & $29.79\mkern2.5mu\rlap{$[29.16\mathord{,}30.48]^{*}$}\phantom{[00.00\mathord{,}00.00]^{*}}$ & $5.27\mkern2.5mu\rlap{$[4.54\mathord{,}5.99]^{*}$}\phantom{[0.00\mathord{,}0.00]^{*}}$ & $3.11\mkern2.5mu\rlap{$[2.48\mathord{,}3.79]^{*}$}\phantom{[0.00\mathord{,}0.00]^{*}}$ & $10.39\mkern2.5mu\rlap{$[9.48\mathord{,}11.24]^{*}$}\phantom{[0.00\mathord{,}00.00]^{*}}$ & $30.54\mkern2.5mu\rlap{$[29.85\mathord{,}31.24]^{*}$}\phantom{[00.00\mathord{,}00.00]^{*}}$ & $29.42\mkern2.5mu\rlap{$[28.78\mathord{,}30.13]^{*}$}\phantom{[00.00\mathord{,}00.00]^{*}}$ \\
 & 40 & $\phantom{\mathord{-}}0.17\mkern2.5mu\rlap{$[\mathord{-}0.27\mathord{,}0.57]^{\phantom{*}}$}\phantom{[\mathord{-}0.00\mathord{,}0.00]^{*}}$ & $30.61\mkern2.5mu\rlap{$[30.08\mathord{,}31.19]^{*}$}\phantom{[00.00\mathord{,}00.00]^{*}}$ & $4.87\mkern2.5mu\rlap{$[4.27\mathord{,}5.48]^{*}$}\phantom{[0.00\mathord{,}0.00]^{*}}$ & $3.35\mkern2.5mu\rlap{$[2.91\mathord{,}3.82]^{*}$}\phantom{[0.00\mathord{,}0.00]^{*}}$ & $\phantom{0}9.72\mkern2.5mu\rlap{$[8.80\mathord{,}10.55]^{*}$}\phantom{[0.00\mathord{,}00.00]^{*}}$ & $32.57\mkern2.5mu\rlap{$[31.98\mathord{,}33.27]^{*}$}\phantom{[00.00\mathord{,}00.00]^{*}}$ & $30.54\mkern2.5mu\rlap{$[29.88\mathord{,}31.15]^{*}$}\phantom{[00.00\mathord{,}00.00]^{*}}$ \\
\addlinespace[3pt]
CAMSL & 5 & $\mathord{-}0.05\mkern2.5mu\rlap{$[\mathord{-}0.33\mathord{,}0.36]^{\phantom{*}}$}\phantom{[\mathord{-}0.00\mathord{,}0.00]^{*}}$ & $\phantom{0}6.55\mkern2.5mu\rlap{$[6.00\mathord{,}7.06]^{*}$}\phantom{[00.00\mathord{,}00.00]^{*}}$ & $3.42\mkern2.5mu\rlap{$[3.09\mathord{,}3.76]^{*}$}\phantom{[0.00\mathord{,}0.00]^{*}}$ & $1.79\mkern2.5mu\rlap{$[1.44\mathord{,}2.14]^{*}$}\phantom{[0.00\mathord{,}0.00]^{*}}$ & $\phantom{0}3.53\mkern2.5mu\rlap{$[3.10\mathord{,}3.91]^{*}$}\phantom{[0.00\mathord{,}00.00]^{*}}$ & $\phantom{0}7.02\mkern2.5mu\rlap{$[6.46\mathord{,}7.55]^{*}$}\phantom{[00.00\mathord{,}00.00]^{*}}$ & $\phantom{0}6.64\mkern2.5mu\rlap{$[6.03\mathord{,}7.13]^{*}$}\phantom{[00.00\mathord{,}00.00]^{*}}$ \\
 & 10 & $\phantom{\mathord{-}}0.32\mkern2.5mu\rlap{$[\mathord{-}0.04\mathord{,}0.71]^{\phantom{*}}$}\phantom{[\mathord{-}0.00\mathord{,}0.00]^{*}}$ & $\phantom{0}8.19\mkern2.5mu\rlap{$[7.51\mathord{,}8.88]^{*}$}\phantom{[00.00\mathord{,}00.00]^{*}}$ & $3.81\mkern2.5mu\rlap{$[3.34\mathord{,}4.35]^{*}$}\phantom{[0.00\mathord{,}0.00]^{*}}$ & $1.70\mkern2.5mu\rlap{$[1.32\mathord{,}2.13]^{*}$}\phantom{[0.00\mathord{,}0.00]^{*}}$ & $\phantom{0}4.28\mkern2.5mu\rlap{$[3.76\mathord{,}4.87]^{*}$}\phantom{[0.00\mathord{,}00.00]^{*}}$ & $\phantom{0}9.10\mkern2.5mu\rlap{$[8.43\mathord{,}9.81]^{*}$}\phantom{[00.00\mathord{,}00.00]^{*}}$ & $\phantom{0}8.26\mkern2.5mu\rlap{$[7.61\mathord{,}8.94]^{*}$}\phantom{[00.00\mathord{,}00.00]^{*}}$ \\
 & 20 & $\phantom{\mathord{-}}0.65\mkern2.5mu\rlap{$[0.33\mathord{,}0.98]^{*}$}\phantom{[\mathord{-}0.00\mathord{,}0.00]^{*}}$ & $\phantom{0}8.98\mkern2.5mu\rlap{$[8.28\mathord{,}9.62]^{*}$}\phantom{[00.00\mathord{,}00.00]^{*}}$ & $3.97\mkern2.5mu\rlap{$[3.56\mathord{,}4.44]^{*}$}\phantom{[0.00\mathord{,}0.00]^{*}}$ & $1.79\mkern2.5mu\rlap{$[1.50\mathord{,}2.09]^{*}$}\phantom{[0.00\mathord{,}0.00]^{*}}$ & $\phantom{0}4.52\mkern2.5mu\rlap{$[4.02\mathord{,}4.95]^{*}$}\phantom{[0.00\mathord{,}00.00]^{*}}$ & $10.55\mkern2.5mu\rlap{$[9.82\mathord{,}11.19]^{*}$}\phantom{[00.00\mathord{,}00.00]^{*}}$ & $\phantom{0}9.15\mkern2.5mu\rlap{$[8.45\mathord{,}9.80]^{*}$}\phantom{[00.00\mathord{,}00.00]^{*}}$ \\
 & 40 & $\phantom{\mathord{-}}0.51\mkern2.5mu\rlap{$[0.27\mathord{,}0.81]^{*}$}\phantom{[\mathord{-}0.00\mathord{,}0.00]^{*}}$ & $\phantom{0}9.00\mkern2.5mu\rlap{$[8.32\mathord{,}9.56]^{*}$}\phantom{[00.00\mathord{,}00.00]^{*}}$ & $3.55\mkern2.5mu\rlap{$[3.20\mathord{,}3.90]^{*}$}\phantom{[0.00\mathord{,}0.00]^{*}}$ & $2.09\mkern2.5mu\rlap{$[1.83\mathord{,}2.34]^{*}$}\phantom{[0.00\mathord{,}0.00]^{*}}$ & $\phantom{0}4.16\mkern2.5mu\rlap{$[3.76\mathord{,}4.58]^{*}$}\phantom{[0.00\mathord{,}00.00]^{*}}$ & $11.58\mkern2.5mu\rlap{$[10.89\mathord{,}12.18]^{*}$}\phantom{[00.00\mathord{,}00.00]^{*}}$ & $\phantom{0}9.41\mkern2.5mu\rlap{$[8.76\mathord{,}10.05]^{*}$}\phantom{[00.00\mathord{,}00.00]^{*}}$ \\
\bottomrule
\end{tabular*}
\endgroup
\tablenotes{Results summarize 200 paired resampling replicates with complete refitting. Entries give mean test weighted-CF differences in percentage points. Brackets show pointwise 95\% robustness intervals, and asterisks mark intervals excluding zero before rounding.}
\end{table*}

Fig.~\ref{fig:cf_performance} shows that P+R generally achieves higher test weighted CF than the standalone clustering methods across the full cohorts. Among these methods, K-modes generally performs best, followed by Z and DTW. These results suggest that the coincidence improvements achieved during fitting also persist on later observations.

The paired resampling results further show that P+R's advantage over methods without coincidence-oriented refinement persists across changes in customer composition and sampled days. All 96 pointwise 95\% robustness intervals for paired test differences against Null, Z, K-modes, DTW, Spearman, and Quantile lie entirely above zero. Mean paired gains range from 3.68 to 30.61 pp against Null and from 1.08 to 3.48 pp against K-modes.

Finally, to examine whether the advantage of P over random initialization observed during training persists on unseen days, we compare P+R with MS+R. P+R achieves a higher mean test weighted CF than MS+R in fourteen of the sixteen settings, with seven robustness intervals entirely above zero.\footnote{Under the same paired evaluation, refinement improves test weighted CF for both Z and K-modes in all sixteen settings, with pointwise 95\% robustness intervals above zero. Intervals for P+R $-$ K-modes+R all contain zero, while those for P+R $-$ Z+R contain zero in twelve settings and favor P+R in four.} These findings indicate that R accounts for much of the improvement in out-of-time coincidence, while P provides a modest additional benefit over random starts in several settings. We therefore retain P, an easy-to-implement and interpretable initialization, as the default initializer.

\subsection{Cluster Composition}
\label{sec:results_composition}

Table~\ref{tab:cluster_composition} compares group composition and refinement-induced membership changes across methods. In Panel A, we summarize group-size balance by the effective group count $K_{\mathrm{eff}}=N^{2}/(\sum_k|C_k|^{2})$, which equals $K$ when all groups have the same size and approaches one as a single group dominates.

\begin{table}[!t]
\centering
\caption{Group composition and refinement-induced membership changes}
\label{tab:cluster_composition}
\begingroup\tablebodyformat
\begin{tabular*}{\columnwidth}{@{\extracolsep{\fill}}l@{\hspace{4pt}}c@{\hspace{4pt}}c@{\hspace{4pt}}c@{\hspace{4pt}}c@{}}
\toprule
Method / diagnostic & GoiEner & SGSC & CKW & CAMSL \\
\midrule
\multicolumn{5}{l}{\textit{A. $K=20$: effective $K$ / largest peak-demand share (\%)}} \\
P & $12.31\;/\;13.9$ & $14.23\;/\;12$ & $7.05\;/\;18.8$ & $14.43\;/\;12.5$ \\
P+R & $16.00\;/\;8.2$ & $16.11\;/\;10$ & $10.89\;/\;10$ & $16.07\;/\;9.5$ \\
Z & $12.56\;/\;10.9$ & $14.89\;/\;10.6$ & $7.06\;/\;17.1$ & $12.38\;/\;12.9$ \\
Spearman & $1.00\;/\;100$ & $1.01\;/\;99.7$ & $1.01\;/\;100$ & $1.01\;/\;100$ \\
Quantile & $1.55\;/\;87.1$ & $5.42\;/\;27.6$ & $3.08\;/\;48.1$ & $5.41\;/\;45.5$ \\
DTW & $13.38\;/\;9$ & $14.15\;/\;12.3$ & $8.84\;/\;12.5$ & $12.10\;/\;13.9$ \\
K-modes & $15.02\;/\;9.5$ & $14.61\;/\;11.6$ & $17.63\;/\;10.4$ & $17.21\;/\;9.8$ \\
MS+R & $17.89\;/\;7.8$ & $16.20\;/\;11.4$ & $11.66\;/\;10.9$ & $16.91\;/\;8.7$ \\
\bottomrule
\end{tabular*}
\smallskip
\begin{tabular*}{\columnwidth}{@{\extracolsep{\fill}}l@{\hspace{4pt}}c@{\hspace{4pt}}c@{\hspace{4pt}}c@{}}
\multicolumn{4}{l}{\textit{B. P+R customer counts: minimum / median / maximum}} \\
\midrule
Dataset & $K=10$ & $K=20$ & $K=40$ \\
GoiEner & $128\;/\;1429.5\;/\;1882$ & $71\;/\;655.5\;/\;1038$ & $25\;/\;325.5\;/\;743$ \\
SGSC & $196\;/\;530\;/\;1066$ & $90\;/\;293\;/\;585$ & $25\;/\;118\;/\;360$ \\
CKW & $108\;/\;324\;/\;1317$ & $24\;/\;158.5\;/\;830$ & $9\;/\;84.5\;/\;345$ \\
CAMSL & $194\;/\;426\;/\;580$ & $47\;/\;200\;/\;405$ & $15\;/\;94.5\;/\;247$ \\
\bottomrule
\end{tabular*}
\smallskip
\begin{tabular*}{\columnwidth}{@{\extracolsep{\fill}}l@{\hspace{4pt}}c@{\hspace{4pt}}c@{\hspace{4pt}}c@{\hspace{4pt}}c@{}}
\multicolumn{5}{l}{\textit{C. $K=10$: membership change after optimal label matching}} \\
\midrule
Method & GoiEner & SGSC & CKW & CAMSL \\
P to P+R & $52.7\%$ & $52.3\%$ & $44.9\%$ & $49.1\%$ \\
\bottomrule
\end{tabular*}
\endgroup
\tablenotes{Results use full-cohort partitions fitted on training data. In Panel~A, peak share is the largest mean group share of summed individual peak demand during testing. Panel~C measures membership changes after matching group labels.}
\end{table}

Table~\ref{tab:cluster_composition} shows that refinement produces less concentrated groups than the initial P clustering. At $K=20$, the effective group count increases from 7.05--14.43 under P to 10.89--16.11 under P+R, while the largest mean test-period share of summed individual peak demand decreases from 12.0--18.8\% to 8.2--10.0\% across the four datasets. Although the minimum group size is one, none of the reported P+R partitions contains singletons. Even at $K=40$, the smallest groups contain 9--25 customers. MS+R produces similarly balanced groups, suggesting that balance alone does not explain P+R's coincidence advantage.\footnote{In contrast, Spearman and Quantile produce highly concentrated partitions at $K=20$, with effective group counts well below the prescribed $K$, which may partly explain their lower weighted CF.} Refinement also substantially changes customer memberships. After matching group labels to maximize agreement, 44.9--52.7\% of customers across the four datasets are assigned to a different group under P+R than under P at $K=10$.

\subsection{Peak-Time Persistence}
\label{sec:persistence}

While weighted CF measures coincidence within groups, it does not show when the customers in each group typically peak, or whether those patterns persist over time. Therefore, in Fig.~\ref{fig:peak_groups} we describe the average peak-time distribution of each P+R group at $K=10$ during both training and testing.

\begin{figure}[!t]
\centering
\includegraphics[width=\columnwidth]{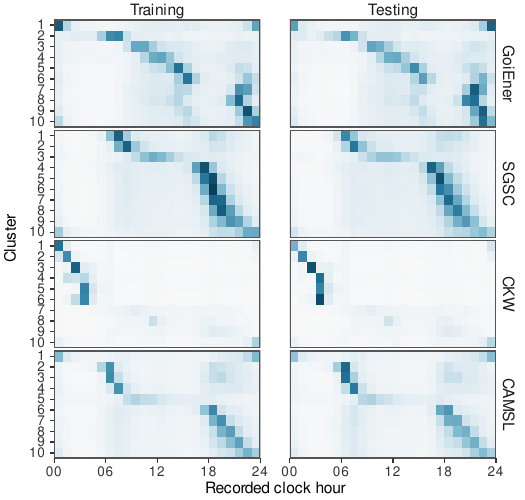}
\caption{Training and test peak-time patterns of P+R groups ($K=10$).}
\label{fig:peak_groups}
\figurenotes{Cells show group averages of customers' estimated probabilities of peaking during the indicated recorded-clock hour. Darker shades indicate higher values. Memberships and row order are fixed from training. Color scales are shared within datasets but differ across datasets.}
\end{figure}

The groups exhibit distinct peak-time patterns, including overnight concentrations in CKW and contrasting morning- and evening-oriented groups in SGSC and CAMSL. Despite some variation in timing and concentration, these patterns remain largely recognizable during testing. Because memberships and group order are fixed from the training period, the test panels are not the result of reclustering; their resemblance to the training panels indicates persistence in the groups' average peak-time patterns.

\subsection{Sensitivity to Minimum Cluster Size}
\label{sec:results_minimum_size}

Finally, we examine how minimum group-size requirements affect the coincidence achieved by P+R. Note that, although these requirements can support meaningful aggregation units, they also restrict the partitions available for optimization. To assess this trade-off, we impose the same minimum size $L_k=L$ on every group across the four full cohorts at $K\in\{5,10,20,40\}$. We compare the baseline $L=1$ with minimum sizes corresponding to 10\%, 25\%, 50\%, 75\%, and 100\% of the average group size $N/K$, rounded down to at least one customer. These fractions are denoted by $\alpha\in\{0.10,0.25,0.50,0.75,1.00\}$. Fig.~\ref{fig:minimum_size} shows the changes in training weighted CF and test weighted CF relative to $L=1$.\footnote{For $K\in\{20,40\}$ and $L_k=\lfloor N/K\rfloor$, refinement took 33.4--2912.8~s and converged after 13--49 accepted updates on the same machine using six threads, illustrating the increased assignment cost when $M$ approaches $N$.}

\begin{figure}[!t]
\centering
\includegraphics[width=\columnwidth]{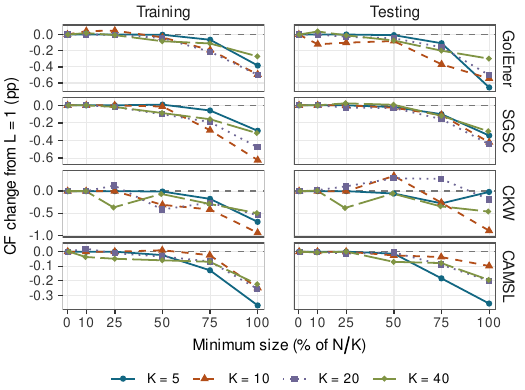}
\caption{Changes in weighted CF with minimum group size.}
\label{fig:minimum_size}
\figurenotes{Changes are relative to $L=1$, with minimum sizes expressed as percentages of $N/K$. The same 30 P initialization seeds are reused across minimum-size settings for each dataset and $K$. All partitions are fitted on training data, with memberships fixed during testing.}
\end{figure}

Fig.~\ref{fig:minimum_size} shows that moderate minimum-size requirements have little effect on out-of-time coincidence. At $\alpha=0.50$, the largest reduction in test weighted CF across the sixteen dataset--$K$ settings is only 0.081 pp. In some settings, test performance even improves despite a lower training weighted CF, suggesting a possible regularizing effect---although this behavior is not consistent.

As the minimum size approaches the average group size, the constraints become more restrictive. At $\alpha=1.00$, group sizes are nearly equal, and both training weighted CF and test weighted CF decrease in all sixteen settings. Nevertheless, test losses remain between 0.018 and 0.893 pp, with a median of 0.350 pp. Notably, even under these constraints, P+R achieves a higher test weighted CF than K-modes with $L=1$ in all sixteen settings, with gains of 0.200--2.772 pp. Thus, nearly equal group sizes can be imposed at a modest cost in out-of-time coincidence while retaining an advantage over the strongest standalone benchmark.

\section{Conclusion}
\label{sec:conclusions}

Peak-demand coincidence within customer groups depends on the temporal alignment of individual loads, which conventional similarity-based clustering does not directly optimize. We address this gap with a peak-demand-weighted CF objective and a fixed-peak refinement procedure (R) that solves each reassignment step exactly and strictly increases weighted CF at every accepted update. Initialized with peak-time-distribution clustering (P), refinement improves weighted CF during fitting in every setting and makes the outcome less dependent on the starting partition. On later observations, the resulting P+R method outperforms all standalone benchmarks on four public smart-meter panels, with mean paired gains of 1.08--3.48 pp over the strongest. Refinement from random starts achieves most of this improvement, so R accounts for the bulk of the gain, while P adds a modest further benefit. The refined groups also retain recognizable peak-time patterns. Moreover, requiring nearly equal group sizes lowers test weighted CF by less than one pp in every setting.

Beyond these empirical findings, the resulting coincidence-oriented groups offer potential applications in distribution planning and demand-side programs. In planning, they could serve as coincidence groups for feeder and transformer service areas. Combined with network connectivity, demand magnitudes, and information on the temporal overlap of group loads, the composition and coincidence factors of these groups could inform assessments of peak-demand diversity and estimates of asset peak demand. In demand response, recurring peak-time patterns, combined with information on demand-side flexibility, could help identify which customer groups to target and during which periods. Future work could evaluate these applications by testing whether coincidence-based aggregation improves peak-demand estimates in realistic networks and by incorporating the groups into network-aware flexibility studies.

\bibliographystyle{IEEEtran}
\bibliography{Refs_revised}
\end{document}